\documentclass[nohyperref]{article}

\usepackage{microtype}
\usepackage{graphicx}
\usepackage{subfigure}
\usepackage{booktabs} %

\usepackage{hyperref}

\usepackage{caption}
\usepackage[accepted]{icml2022}

\usepackage{amsmath}
\usepackage{amssymb}
\usepackage{mathtools}
\usepackage{amsthm}

\usepackage[capitalize,noabbrev]{cleveref}

\usepackage{amsfonts}
\usepackage[ruled,algo2e]{algorithm2e}
\usepackage{svg}
\usepackage{enumerate}
\usepackage{paralist}
\usepackage{multirow}
\usepackage{colortbl}
\usepackage{dsfont}

\usepackage{tikz}
\usepackage{tkz-graph}
\usetikzlibrary{shapes.geometric}
\usetikzlibrary{backgrounds}
\usetikzlibrary{arrows.meta}
\usepackage{placeins}
\usepackage{xcolor}
\definecolor{mygray}{gray}{0.9}

\usepackage{amsmath, mathtools}
\usepackage{centernot}
\usepackage{amssymb}
\usepackage{graphicx,subfigure}
\usepackage{amsthm}
\usepackage{bbm}
\usepackage{tabu}
\usepackage{paralist}
\usepackage{enumitem}
\usepackage{physics}
\usepackage{accents}
\usepackage[english]{babel}
\usepackage[utf8x]{inputenc}
\usepackage{amsfonts}
\usepackage{booktabs}
\usepackage{tabu}
\usepackage[T1]{fontenc}
\usepackage{mathrsfs}
\usepackage{multirow}
\usepackage{stmaryrd}
\usepackage{tikz}
\usetikzlibrary{backgrounds}
\pgfdeclarelayer{bg1}
\pgfdeclarelayer{bg2}
\pgfsetlayers{bg2, bg1, main}
\usepackage{xcolor}
\usetikzlibrary{arrows.meta}
\usepackage{pifont}%
\newcommand{\cmark}{\ding{51}}%
\newcommand{\xmark}{\ding{55}}%

\newcommand\Item[1][]{%
  \ifx\relax#1\relax  \item \else \item[#1] \fi
  \abovedisplayskip=0pt\abovedisplayshortskip=0pt~\vspace*{-\baselineskip}}

\newcommand{\inv}{\operatorname{inv}}
\newcommand{\supp}{\operatorname{supp}}

\newcommand{\HSIC}{\operatorname{HSIC}}

\newcommand{\ci}{\mathrel{\perp\mspace{-10mu}\perp}}

\newcommand{\nci}{\centernot{\ci}\hspace{-2pt}}
\DeclareMathOperator{\EX}{\mathbb{E}}%
\DeclareMathOperator{\E}{\mathbb{E}}%

\DeclareMathOperator*{\argmin}{arg\,min}
\newcommand{\Cov}{\mathrm{Cov}}

\newcommand{\simiid}{\overset{\text{i.i.d.}}{\sim}}

\newcommand{\independent}{\mbox{${}\perp\mkern-11mu\perp{}$}}

\renewcommand{\P}{\mathbb{P}}

\newcommand{\R}{\mathbb{R}}

\let\do\relax

\SetKwFor{RepTimes}{repeat}{times}{end}
\DontPrintSemicolon

\theoremstyle{plain}
\newtheorem{theorem}{Theorem}[section]
\newtheorem{proposition}[theorem]{Proposition}

\theoremstyle{definition}

\theoremstyle{remark}

\newcounter{subassumption}[asu]

\makeatletter
\renewcommand{\p@subassumption}{\theasu}%
\makeatother

\usepackage[textsize=tiny]{todonotes}

\icmltitlerunning{Exploiting Independent Instruments}

\begin{document}

\twocolumn[
\icmltitle{Exploiting Independent Instruments: Identification and Distribution Generalization}

\icmlsetsymbol{equal}{*}

\begin{icmlauthorlist}
\icmlauthor{Sorawit Saengkyongam}{KU}
\icmlauthor{Leonard Henckel}{KU}
\icmlauthor{Niklas Pfister}{KU}
\icmlauthor{Jonas Peters}{KU}

\end{icmlauthorlist}

\icmlaffiliation{KU}{Department of Mathematical Sciences, University of Copenhagen, Denmark}

\icmlcorrespondingauthor{Sorawit Saengkyongam}{ss@math.ku.dk}

\icmlkeywords{Causality, Distribution Generalization, Instrumental Variables, Independence, Causal Discovery}

\vskip 0.3in
]

\printAffiliationsAndNotice{}  %

\begin{abstract}
    Instrumental variable models allow us to identify a causal function between covariates $X$ and a response $Y$, even in the presence of unobserved confounding. Most of the existing estimators assume that the error term in the response $Y$ and the hidden confounders are
    uncorrelated with the instruments $Z$. This is often motivated by a graphical separation, an argument that also justifies independence.
    Positing an independence restriction, however, leads to strictly stronger identifiability results. We connect to the existing literature in econometrics and provide a practical method called HSIC-X for exploiting independence that can be combined with any gradient-based learning procedure. We see that even in identifiable settings, taking into account higher moments may yield better finite sample results. 
    Furthermore, we exploit the independence for distribution generalization.
    We prove that the proposed estimator is invariant to distributional shifts on the instruments and worst-case optimal whenever these shifts are sufficiently strong. These results hold even in the under-identified case where the instruments are not sufficiently rich to identify the causal function.
\end{abstract}

\section{Introduction} \label{sec:intro}
When estimating the causal function between a vector of covariates $X$ and a response $Y$ in the presence of unobserved confounding,   standard regression procedures
such as ordinary least squares (OLS)
are even asymptotically biased.
Instrumental variable approaches 
\citep{Wright1928, angrist1995identification, Newey2013}
exploit the existence of exogenous heterogeneity in the form of an instrumental variable (IV) $Z$ and estimate, under suitable conditions, the causal function consistently. 
Importantly, 
the errors in $Y$ and the hidden confounders $U$ should be uncorrelated with the instruments $Z$.
Usually, this has to be argued for with background knowledge. Often this is done by assuming that the data generating process follows a structural causal model (SCM) \citep{Pearl2009, Bongers2021} (so that the distribution is Markov with respect to the induced graph), and that $Y$ and $U$ are $d$-separated from $Z$ in the graph obtained by removing the edge from $X$ to $Y$ (this is the case for the DAG at the beginning of Section~\ref{sec:identifiability}, for example). In particular, this requires an argument that $Z$ is not causing $Y$ directly but only via $X$. 
But this argument does not only imply that the errors in $Y$ and $U$ are uncorrelated but also that they are independent. This independence
comes with several benefits.

For example, 
even in settings where the causal function can be identified by classical approaches based on uncorrelatedness, 
it has been observed that
the independence can be exploited to construct estimators
that achieve the semiparametric efficiency bound, at least when 
the error distribution comes from a known, parametric family \citep{Hansen2010}.
Furthermore, the independence restriction is stronger
than uncorrelatedness and therefore yields stronger 
identifiability results, 
which has been reported in the field of econometrics \citep[e.g.,][]{Imbens2009, Chesher2003}. 
For example, even binary instruments may be able to identify nonlinear effects \citep{Dunker2014, Torgovitsky2015, Loh2019}.

In this work, we investigate the independence restriction in more detail: we add to existing identifiability results, provide methods for exploiting the restriction for finite data and analyse implications for distribution generalisation.

More precisely, in
Section~\ref{sec:identifiability} we
discuss the identifiability conditions in a general and simple form, list some of the existing identifiability results that they imply, add novel results to this list, and extend the framework to conditional IV.

We also provide a practical method that exploits the above independence for estimating causal effects from data.
It relies on the Hilbert-Schmidt independence criterion (HSIC) \citep{Gretton2008} which has become a widely used tool for testing independence in a joint distribution. Equipped with a characteristic kernel (such as the Gaussian kernel),
HSIC is positive and equals zero if and only if the considered distribution factorizes.
We propose an easy-to-use method, called HSIC-X (`X' for `exogenous'), %
that minimizes HSIC directly.
The underlying problem is non-convex 
but can be tackled using widely used methods from stochastic optimization.
While theoretical guarantees are hard to obtain, it has recently been shown empirically that 
reliable optimization of HSIC seems possible, at least when considering the different problem of minimizing independence of residuals and predictors in a regression problem \citep{greenfeld20a, Mooij2009icml}.
Furthermore, the optimization can be initialized at informative starting points, such as the two-stage least squares (2SLS) solution, or even the OLS solution. 
HSIC-X can be combined with any machine learning method
for nonparametric regression that is optimized by (stochastic) gradient descent. 
The details of our method are described in Section~\ref{sec:algorithm}.

Furthermore,  
the independence restriction 
can be exploited for 
distribution generalization. In Section~\ref{sec:generalization},
we construct an estimator HSIC-X-pen (`pen' for `penalization')
that is 
worst-case optimal in nonlinear settings under distributional shifts
corresponding to interventions on $Z$.
This is particularly interesting in underidentified settings, where the causal function cannot be identified from the data. 
Our work thereby adds to an increasing literature connecting 
distributional robustness and causal inference \citep[e.g.,][]{Schoelkopf2012icml, Rojas2016, Magliacane2018, Rothenhaeusler2018, Arjovsky2019, Christiansen2020DG, Pfister2019stab, Yuan2021, pmlrkrueger21a, pmlrcreager21a}.
As for HSIC-X, the estimator HSIC-X-pen is modular in that it can be used
with any machine learning method for nonparametric regression.

Experiments on both simulated and real world data in
Section~\ref{sec:exp}
confirm that HSIC-X can exploit the improved identifiability guarantees and
can be more efficient in finite samples if wrong solutions yield  
both second and higher order dependencies between the residuals and the instruments.
The code for all the experiments is available at \url{https://github.com/sorawitj/HSIC-X}. %
All proofs are provided in Appendix~\ref{sec:proofs}.

\subsection{Further Related Work}
Independence between $Z$ and 
$Y - f(X)$ 
can also be characterized differently; e.g., it
is equivalent to 
\begin{equation*} %
    \E [\eta(Z) \psi(Y - f(X))] = \E [\eta(Z)] \E[\psi(Y - f(X))]
\end{equation*}
for all $\eta$ and $\psi$ in a sufficiently rich class (such as bounded continuous functions).
This suggests 
estimating the causal function by
a generalized methods of moments (GMM) 
approach. 
Indeed, \citet{Poirier2017} 
focuses on the derivation of  optimal weights for such estimating equations and the asymptotic behaviour of the estimator. 
\citet{Dunker2014,Dunker2021} phrase the problem as an inverse problem and derive convergence rates for the  iteratively regularized Gauss-Newton method.
The above works focus on theoretical advancements (e.g., the experiments are restricted to univariate settings). 
Considering the independence of residuals has also been suggested, by \citet[][Section~5]{Peters2016jrssb} but no practical method was provided.
Powerful regression techniques and machine learning methods to measure dependence in IV settings exist \citep{Hartford2017, Singh2019, Bennett2019, Muandet2020} but to the best of our knowledge, none of these methods exploit the independence restriction.

In summary, exploiting independence  does not seem to have played a major role in practice. Arguably, one of the reasons is that independence restrictions are difficult to work with in theory and practice. E.g., choosing the class of functions $\psi$ 
is non-trivial. %

\section{Identifiability from Higher Order Moments}
\label{sec:identifiability}
Unless stated otherwise, we consider the following SCM\\
\begin{minipage}{0.49\linewidth}
  \begin{align*}
    Z&\coloneqq \epsilon_Z\\
    U&\coloneqq \epsilon_U\\
    X&\coloneqq g(Z, U, \epsilon_X)\\
    Y&\coloneqq f(X)  + h(U, \epsilon_Y)
  \end{align*}
  \hspace{0.5em}
\end{minipage}%
\begin{minipage}{0.49\linewidth}
  \resizebox{0.95\textwidth}{!}{
    \begin{tikzpicture}[scale=1.4]
      \tikzstyle{VertexStyle} = [shape = circle, minimum width = 3em,
      fill=lightgray] \Vertex[Math,L=Y,x=1,y=0]{Y}
      \Vertex[Math,L=X,x=-1,y=0]{X}
      \Vertex[Math,L=Z,x=-2,y=1]{Z}
      \tikzset{EdgeStyle/.append style = {-Latex, line width=1}}
      \Edge[label=$f$](X)(Y)
      \Edge[label=$g$](Z)(X)
      \tikzstyle{EdgeStyle}=[bend left=50, line width=1, Latex-Latex, dashed]
      \Edge[label={\Large $U$}](X)(Y)
    \end{tikzpicture}}
\end{minipage}
where $(\epsilon_Z, \epsilon_U, \epsilon_X, \epsilon_Y)\sim Q$ are jointly independent noise variables, $Z\in\mathbb{R}^r$ 
are \emph{instruments}, $U\in\mathbb{R}^q$ are unobserved variables, $X\in\mathbb{R}^d$ are \emph{predictors} and $Y\in\mathbb{R}$ is a \emph{response}. For simplicity, we additionally assume that $\E[\epsilon_Z]=\E[X]=\E[h(U, \epsilon_Y)]=0$. We call 
a collection $M=(f, g, h, Q)$ an \emph{IV model},
with $f \in \mathcal{F} \subseteq \{f_{\diamond}\,|\,f_{\diamond}:\mathbb{R}^d\rightarrow\mathbb{R}\}$ 
and
$\mathcal{F}$  a pre-specified function class. 
The collection of all IV models of this form is denoted by $\mathcal{M}$. Any IV model $M\in\mathcal{M}$ induces a distribution $\P_M$ over the observed variables $(X, Y, Z)$.
We denote the data generating IV model by $M^0=(f^0, g^0, h^0, Q^0)$ (parts of the data are unobserved). We refer to the function $f^0$ as the \emph{causal function}. We can use it to compute the causal effect of any treatment contrast, such as the average treatment effect $f(1)-f(0)$ in case of a binary $X$. We assume that the causal function $f(\cdot)$ only depends on $X$, i.e., is homogeneous, and thus we do not need to distinguish between local and global treatment effects \citep[e.g][]{angrist1995identification}. Unless stated otherwise, we also assume that $f^0$ can be parameterized such that $f^0(\cdot) = \phi(\cdot)^\top \theta^0$ for some $\theta^0 \in \Theta \subseteq \R^p$, with a (known) function $\phi: \R^d \rightarrow \R^p$ 
and, for simplicity, we assume that $\theta \neq {\theta^0}$ implies $\phi(\cdot)^\top \theta \neq \phi(\cdot)^\top {\theta^0}$. In practice the true basis $\phi$ does not need to be known and can be approximated by a sufficiently flexible function approximator (e.g., neural networks), see Section~\ref{sec:algorithm}.

Let us assume that we observe $n$ i.i.d.\ observations from $(X, Y, Z)\sim\P_{M^0}$. In this paper, we consider two  
problems: 
estimating the causal function $f^0$ and 
predicting the response $Y$ under distributional shifts on the variables $Z$ (see Section~\ref{sec:generalization}). 
In order to solve either of these tasks, we make use of (parts of) the causal function $f^0$. We therefore require that it is uniquely determined by the observed distribution $\P_{M^0}$ (we relax this condition in Section~\ref{sec:generalization}). This is often termed \emph{identifiability} of the causal function.

\begin{table*} 
\caption{Overview of some of the identifiability results, described in Section~\ref{sec:identifiability}. \label{tab:identif}}
\centerline
{\small
 \begin{tabular}{cccccl}
    \toprule    
    \multirow{2}{*}{$|\supp(Z)|$} & {$Z$ acts on} & {$Z$ acts on} & {identif. with~\eqref{eq:moment1}} & {identif. with~\eqref{eq:moment2}} & \multirow{2}{*}{Comments} \\
    & mean of $X$ & higher orders of $X$ &possible&possible& \\
    \toprule
    \rowcolor{mygray}
    \multirow{1}{*}{$\infty$} & \multirow{1}{*}{yes}   & \multirow{1}{*}{yes/no} & \multirow{1}{*}{\cmark} & \multirow{1}{*}{\cmark} & gain efficiency with \eqref{eq:moment2}, cf Sec.~\ref{sec:sim_exp}\\
    $\infty$ & no & yes &\xmark & \cmark & cf Prop~\ref{prop:stronger_identifability}-iii-a\\
    \rowcolor{mygray}
     &&&&&  identif.\ $\leq m$ par.\ with \eqref{eq:moment1}, gain efficiency \\
     \rowcolor{mygray}
    \multirow{-2}{*}{$m<\infty$} & \multirow{-2}{*}{yes}   & \multirow{-2}{*}{yes/no} & \multirow{-2}{*}{\cmark} & \multirow{-2}{*}{\cmark} &  \& identif. with \eqref{eq:moment2},  cf Prop~\ref{prop:stronger_identifability}-iii-b, Sec~\ref{sec:sim_exp}\\
    \multirow{1}{*}{$m<\infty$} & \multirow{1}{*}{no}   & \multirow{1}{*}{yes} & \multirow{1}{*}{\xmark} & \multirow{1}{*}{\cmark} & cf example in proof of Prop~\ref{prop:no-simplification}\\
    \bottomrule
  \end{tabular}
  }
\end{table*}
In a classical IV approach, 
identification of $f^0$ (or, equivalently, $\theta^0$) is based 
on the \emph{moment restriction} 
\begin{equation} \label{eq:moment1}
    \E [\eta(Z)(Y - \phi(X)^\top \theta)] = 0,
\end{equation}
where $\eta: \R^r \rightarrow \R^k$  is a known function.
Under the IV model $M^0$ this condition is satisfied for $\theta^0$. A sufficient and necessary condition for identifiability based on \eqref{eq:moment1} is given by the \emph{(moment) identifiability condition}
\begin{equation} \label{eq:moment1b}
    \E [\eta(Z)\phi(X)^\top] \tau = 0 \quad \Longrightarrow \quad \tau = 0.
\end{equation}
This condition is sometimes called the \emph{rank condition} because it is equivalent to the  matrix $\E [\eta(Z)\phi(X)^\top]$ having rank $p$, which, in particular, implies that $k \geq p$
\citep[e.g.,][]{wooldridge2010econometric}.

The moment restriction~\eqref{eq:moment1} aims to detect 
mean shifts of the residuals when varying the value of $Z$. 
For example,~\eqref{eq:moment1}
can only identify $\theta^0$ if we do not have
for all $i \in \{1, \ldots, p\}$  
$\E[\phi_i(X)|Z] = 0$ almost surely 
(otherwise the function 
$x \mapsto f^0(x) + \phi_i(x)$ solves~\eqref{eq:moment1}, too).
But even if there are no mean shifts and~\eqref{eq:moment1} is not powerful enough to identify $f^0$, we may still be able to identify $f^0$ from 
a stronger condition. To this end, we consider
the \emph{independence restriction}
\begin{equation} \label{eq:moment2}
    Z \ci Y - \phi(X)^\top \theta.
\end{equation}
In this paper, we consider 
identifying $f^0$ based on~\eqref{eq:moment2}, rather than~\eqref{eq:moment1}.
Under the IV model $M^0$ described above, this condition is satisfied for the true parameter $\theta^0$.
A necessary and sufficient condition for identifiability using~\eqref{eq:moment2} is   
\begin{equation}
    \label{eq:moment2c}
    Z \ci h(U, \epsilon_Y) + \phi(X)^\top \tau \quad \Longrightarrow \quad \tau = 0.
\end{equation}
This, in particular, implies that for all $i \in \{1, \ldots, p\}$
we have 
$Z \nci  h(U, \epsilon_Y) + \phi_i(X)$ (since, otherwise $\tau = e_i$ would violate the above condition). 

If condition~\eqref{eq:moment2c} is violated, condition~\eqref{eq:moment1b} is violated, too, but the implication does not hold in the other direction. 
As a result, the independence restriction yields strictly stronger identifiability results (see also 
Table~\ref{tab:identif}). 

\begin{proposition}[Identifiability based on independence]
\label{prop:stronger_identifability}
Consider an IV model $M^0$ and assume 
$f^0(\cdot)=\phi(\cdot)^\top \theta^0$ for some $\theta^0 \in \mathbb{R}^{p}$. 
Then, the following statements hold.
\begin{compactitem}
\item[(i)] If $\theta^0$ is identifiable from the moment restriction~\eqref{eq:moment1} it is also identifiable from the independence restriction \eqref{eq:moment2}. 
\item[(ii)]
There exist IV models such that $\theta^0$ is identifiable from the independence restriction \eqref{eq:moment2} but not 
from~\eqref{eq:moment1}. 
\item[(iii)]
In particular, there are examples of the following type that satisfy the conditions of (ii). (a) `Nonadditive $Z$': 
\begin{align} \label{eq:scm_iv}
\begin{split}
    Z \coloneqq \epsilon_Z  & \qquad 
    U \coloneqq \epsilon_U \qquad
    X \coloneqq ZU + \epsilon_X \qquad \\
    & Y \coloneqq X + U + \epsilon_Y,
    \end{split}
\end{align}
with $\mathcal{F} = \{f\,|\, f(x) = \theta x\}$, 
where $(\epsilon_Z, \epsilon_U, \epsilon_X, \epsilon_Y)$ are  jointly independent standard Gaussian variables.
(b) `Binary $Z$': $Z \in \{0,1\}$, $X \in \R$, $p>2$ (i.e., $\mathcal{F}$ contains nonlinear functions). 
(c) `Independent $Z$': $Z \ci X$. 
\end{compactitem}
\end{proposition}
Statements (i) and (ii) are known 
\citep[e.g.,][]{Imbens2009, Chesher2003}
but for completeness we nevertheless include their proofs in Appendix~\ref{sec:proofs}. 
Intuitively, the causal function from the SCM \eqref{eq:scm_iv} in (a) is not identifiable from the moment restriction because the instrument (or any transformation $\eta(Z)$) does not correlate with the mean of $X$, i.e., $\E[\eta(Z)X]=\E[\eta(Z)\E[X|Z]]=0$. Therefore, for any $\gamma\in\R$, the shifted causal function $f(x)\coloneqq f^0(x)+\gamma x$, for $x\in\R$, also satisfies the moment restriction. 
In the proof of Proposition~\ref{prop:stronger_identifability}, we argue that identifiability using~\eqref{eq:moment1b} is impossible for an example of type (b). It may come as a surprise that it is indeed possible to identify  nonlinear functions even if the instrument $Z$ has a discrete support with small cardinality. This observation has been reported, e.g., by \citet{Dunker2014, Torgovitsky2015, Loh2019}. 

Even in cases where the causal function is identifiable from both~\eqref{eq:moment2c} and~\eqref{eq:moment1b}, it may be beneficial to consider the independence restriction. This is, for example, the case when the effect of the instrument can be seen in both the conditional mean of $\phi(X)$, given $Z$, and in higher moments. In our simulation experiments in Section~\ref{sec:exp}, we see that taking into account dependencies in higher moments may yield better statistical performance.

Condition~\eqref{eq:moment2c} depends on the unknown $h(U,\epsilon_Y)$; 
at first glance, one might believe 
that the hidden term can be dropped from condition~\eqref{eq:moment2c} but this is not the case.
\begin{proposition} \label{prop:non-reducible}
Consider the IV model $M^0$ and assume that $\phi(\cdot)=(\phi_1(\cdot),\dots,\phi_p(\cdot))$ is a collection of basis functions such that $f^0(\cdot)=\phi(\cdot)^\top \theta^0$ for some $\theta^0\in \mathbb{R}^{p}$. Then, in general, condition~\eqref{eq:moment2c} does not imply
\begin{equation}
    Z \ci \phi(X)^\top \tau \quad \Longrightarrow \quad \tau = 0.
    \label{eq:moment2d}
\end{equation}
Furthermore,~\eqref{eq:moment2d} does not imply~\eqref{eq:moment2c}, either. 
\label{prop:no-simplification}
\end{proposition}
Identifiability condition~\eqref{eq:moment1b} depends only on the joint distribution of $(Z,X)$ (which can be estimated from data), but we do not consider this an advantage: in practice, it is usually desirable to consider empirical relaxations of the identifiability conditions
and output the set of $\theta$'s that (approximately) satisfy the empirical version of~\eqref{eq:moment1} or~\eqref{eq:moment2}, respectively. For example, one can invert statistical tests to construct confidence sets, see Section~\ref{sec:confint}.

\subsection{Conditional Instrumental Variables}
\label{sec: CIV methodology}
In some applications, such as the one we consider in Section~\ref{sec:real_exp}, restrictions~\eqref{eq:moment1} and~\eqref{eq:moment2}
may be violated for the causal parameter, e.g., because there are confounding variables $W$ between the instruments $Z$ and the response $Y$. 
Under certain assumptions, however, the framework of conditional IV (CIV) \citep{frolich2007nonparametric,Newey2013} still allows us to identify the causal function $f^0$ from the observed distribution. 
Both the identifiability point of view and the methodology we develop in Section~\ref{sec:algorithm} can be extended to CIV. More details are provided in Appendix~\ref{sec:app_implmentation_civ}. 
Consider the following SCM.

\begin{minipage}{0.5\linewidth}
  \begin{align*}
    W&\coloneqq m(\epsilon_W,V,U)\\
    Z&\coloneqq q(W,V,\epsilon_Z)\\
    X&\coloneqq g(Z, W, U, \epsilon_X)\\
    V&\coloneqq \epsilon_{V}, \; 
    U \coloneqq \epsilon_{U}\\
    Y&\coloneqq f(X) + h(W,U, \epsilon_Y)
  \end{align*}
  \hspace{0.5em}
\end{minipage}%
\hspace{-0.08\linewidth}
\begin{minipage}{0.56\linewidth}
\vspace{-0.08\linewidth}
  \resizebox{0.94\textwidth}{!}{
    \begin{tikzpicture}[scale=1.4]
      \tikzstyle{VertexStyle} = [shape = circle, minimum width = 3em,
      fill=lightgray] \Vertex[Math,L=Y,x=1,y=0]{Y}
      \Vertex[Math,L=X,x=-1,y=0]{X}
      \Vertex[Math,L=Z,x=-3,y=0]{Z}
      \Vertex[Math,L=W,x=-1,y=1.5]{W}
      \tikzset{EdgeStyle/.append style = {-Latex, line width=1}}
      \Edge[label=$f$](X)(Y)
      \Edge[label=$g$](Z)(X)
      \Edge[label=$q$](W)(Z)
      \Edge[label=$h$](W)(Y)
      \Edge[label=$g$](W)(X)
      \tikzstyle{EdgeStyle}=[bend right=50, line width=1, Latex-Latex, dashed]
      \Edge[label={\Large $U$}](X)(Y)
     \tikzstyle{EdgeStyle}=[bend left=50, line width=1, Latex-Latex, dashed]
     \Edge[label={\Large $V$}](Z)(W)
      \tikzstyle{EdgeStyle}=[bend left=50, line width=1, Latex-Latex, dashed]
      \Edge[label={\Large $U$}](W)(Y)
    \end{tikzpicture}}
\end{minipage}
where $(\epsilon_W,\epsilon_Z, \epsilon_{V}, \epsilon_{U}, \epsilon_X, \epsilon_Y)$ are jointly independent  %
and 
$W\in\mathbb{R}^t$ are observed covariates. %
We assume that 
there is no edge from $V$ to $W$ or 
no edge from $U$ to $W$,
i.e., either $Z$ and $W$ or $W$ and $Y$ are not confounded. 

Due to the unobserved confounding, 
\eqref{eq:moment2} may not hold 
but as we assume that there is either no confounding between $Z$ and $W$ or between $W$ and $Y$, we can instead use the \textit{conditional independence restriction}
\begin{equation}
    Z \ci Y-\phi(X)^\top \theta \,| \,W
    \label{condition: CIS 1}
\end{equation}
to identify $f^0$. A corresponding sufficient and necessary identifiability condition for $f^0$ is 
\begin{equation}
    Z \ci h(W,U,\epsilon_Y) + \phi(X)^\top \tau \,| \,W \, \, \implies \tau=0.
    \label{condition: ident. CIS 1}
\end{equation}
We can estimate $f^0$ using a loss that is minimized if restriction~\eqref{condition: CIS 1} is satisfied, e.g., using a conditional independence measure \citep[e.g.][]{Fukumizu2008,Zhang2011uai,berrett2019nonparametric,Shah2018}.
In Appendix~\ref{sec:app_implmentation_civ} we discuss cases, in which we can avoid using a conditional independence restriction.

\section{Independence-based IV with HSIC}\label{sec:algorithm}
Starting from the independence restriction~\eqref{eq:moment2}, our goal is to find a function $\hat{f}$ such that the residuals $R^{\hat{f}} \coloneqq Y - \hat{f}(X)$ are independent of the instruments $Z$. In the identifiable case, that is, if condition~\eqref{eq:moment2c} is satisfied, only the causal function $f^0$ achieves independence.
Thus, given an i.i.d.\ sample $(x_i, y_i, z_i)_{i=1}^n$ 
of the variables $(X, Y, Z)$, our method aims to find a function $\hat{f}$ that minimizes the dependency between the residuals $(r_i^{\hat{f}})_{i=1}^n$, with $r_i^{\hat f} \coloneqq y_i - \hat{f}(x_i)$, and the instruments $(z_i)_{i=1}^n$.
In this work, 
we measure dependency using HSIC \citep{Gretton2008}.
When using characteristic kernels \citep{Fukumizu2008}, this measure equals zero if and only if the considered joint distribution factorizes.
HSIC has been used for optimization problems before \citep[e.g.,][]{greenfeld20a, Mooij2009icml} and satisfies the conditions used to prove consistency (see Section~\ref{sec:consistency}) but other choices of independence measures are possible, too.
Specifically, we consider the following learning problem:
\begin{equation}
\label{eq:hsic_objective}
    \hat{f} \coloneqq \argmin_{f \in \mathcal{F}} \;  \widehat{\HSIC}((r_i^f , z_i)_{i=1}^n; k_{R^f}, k_Z),
\end{equation}
where $\widehat{\HSIC}((r_i^f , z_i)_{i=1}^n; k_{R^f}, k_Z) \coloneqq \trace(KHLH)$ is a 
consistent estimator of $\HSIC((R^f,Z);k_{R^f},k_Z)$ \citep{Gretton2008}; here, $K_{ij} = k_{R^f}(r^f_i, r^f_j)$ and $L_{ij} = k_Z(z_i, z_j)$ are the kernel matrices for the residuals $R^{f}$ and the instruments $Z$ with positive definite kernels $k_{R^f}$ and $k_Z$, respectively, 
 and $H_{ij} := \delta_{ij} - \frac{1}{n}$ is the centering matrix. 
We call the estimator in \eqref{eq:hsic_objective} HSIC-X.

The independence restriction~\eqref{eq:moment2} allows us to learn the causal function $f^0$ up to a bias term (for any $\alpha \in \R$, $\HSIC((Y - f^0(X), Z); k_{R^{f^0}},k_Z)$ and $\HSIC((Y - \alpha - f^0(X), Z); k_{R^{f^0}},k_Z)$ are identical). Nonetheless, we can correct for the bias by using the zero mean assumption of the noise $\epsilon_Y$. The final estimate is then obtained as $\tilde{f}(\cdot) \coloneqq \hat{f}(\cdot) - \frac{1}{n}\sum_{i=1}^n( y_i - \hat{f}(x_i))$.

\subsection{Regularizing towards Predictive Functions}\label{sec:hsic_pen}
\label{sec:HSIC-X-Pen} 

In many practical applications 
the identifiability condition~\eqref{eq:moment2c} 
(or~\eqref{eq:moment1b} for classical IV) 
is satisfied 
but many parameters approximately solve the empirical version of~\eqref{eq:moment2}.
This is the case, for example, if the influence of the instruments is weak, which usually yields subpar finite sample properties.
Furthermore, classical estimators like 2SLS
are known to only have moments up to the degree of over-identification \citep[e.g.,][]{Mariano2001}.
To stabilize the estimation it has been proposed to regularize towards a predictive function, such as the OLS in linear settings, see, e.g., the K-class estimators \citep{Theil1958,  Jakobsen2020}, which contain the OLS, 2SLS, the FULLER \citep{Fuller1977} and LIML estimators \citep{AndersonRubin} as special cases.

We propose an analogous regularization for our estimator 
and call this variant HSIC-X-pen. More specifically, for a 
convex loss function $\ell:\R\rightarrow\R$ we modify the optimization problem~\eqref{eq:hsic_objective} as follows,
\begin{align} 
    \hat{f}^{\lambda}=\argmin_{f\in\mathcal{F}}\, 
    \widehat{\HSIC}&((r_i^{f}, z_i)_{i=1}^n; k_{R^f},k_Z)
    \label{eq:penalized_HSIC_IV}\\
    &\qquad + 
    \lambda \textstyle\sum_{i=1}^n\ell(y_i-f(x_i)),
    \nonumber
\end{align}
where $\lambda\in[0, \infty)$ is a tuning parameter. 
Unlike in the linear settings described above \citep[e.g.,][]{Fuller1977}, deriving a data-driven choice of the tuning parameter $\lambda$ is non-trivial. We propose to select $\lambda$ following a procedure analogue to the one described in \citet{Jakobsen2020}: we select the largest possible value of $\lambda$ for which an HSIC-based independence test \citep[e.g.,][]{Gretton2008, Pfister2017jrssb} between 
$R^{\hat{f}^{\lambda}}$ and 
$Z$ is not rejected.

As discussed in Section~\ref{sec:generalization}, HSIC-X-pen can be understood in relation to distribution generalization, too. There, one starts from the objective of optimizing the predictive loss and adds the HSIC term as a penalty that regularizes the predictor to guard against distributional shifts.

\subsection{Algorithm and Implementation Details}
We now specify the details of HSIC-X. To solve
\eqref{eq:hsic_objective}, we fix any
parametric function class $\mathcal{F} \coloneqq \{f_{\theta}(\cdot) \mid \theta \in \Theta \subseteq \R^p\}$ (e.g., a linear combination of some basis functions or a neural network)
and optimize the parameters $\theta$ by a gradient-based optimization method. We choose the 
Gaussian kernel $k_{R}$ \citep[e.g.,][]{learningkernels} for the residuals, and the discrete or Gaussian kernel $k_{Z}$ for the instruments depending on whether $Z$ is discrete or continuous, respectively.
The bandwidth parameter $\sigma$ of the
Gaussian kernel is chosen by the median heuristic
\citep[e.g.,][]{Sriperumbudur2009} and is recomputed  during the optimization process (as the residuals change at each iteration).

Since the optimization problem \eqref{eq:hsic_objective} is generally non-convex, the resulting parameter estimates may not be the global optimal solution. We alleviate this problem by introducing the following restarting heuristic. Let $\hat{\theta} \in \Theta$ be a solution to the optimization problem. We conduct an independence test between the resulting residuals $(r_i^{\hat{\theta}} \coloneqq y_i - f_{\hat{\theta}}(x_i))_{i=1}^n$ and the instruments $(z_i)_{i=1}^n$ using HSIC with Gamma approximation \citep{Gretton2008}. We accept the parameters $\hat{\theta}$ if the test is not rejected, otherwise we randomly re-initialize the parameters and restart the optimization. In the spirit of \eqref{eq:penalized_HSIC_IV}, we initialize the parameters in the first trial at the OLS solution. Algorithm~\ref{alg:ind_iv} in Appendix~\ref{sec:app_algo} illustrates the whole optimization procedure with the standard gradient descent update. The gradient step can be replaced with other gradient-based optimization algorithms such as Adam \citep{kingma2014adam} or Adagrad \citep{duchi2011adaptive};
in all of our experiments, we used Adam. 
The Algorithm for HSIC-X-pen is also provided in Appendix~\ref{sec:app_algo}.

\subsection{Consistency}
\label{sec:consistency}
We now prove consistency of the proposed approach in that the
minimizer of~\eqref{eq:hsic_objective}
converges (in probability) against the causal function, as sample size increases.\footnote{
There is a slight mismatch between Theorem~\ref{thm:consistency} and the described algorithm: in practice, the kernel bandwidth is not fixed but is chosen according to the median heuristic. This difference could be accounted for by sample splitting. 
}
\begin{theorem}
Consider the IV model $M^0$, assume that  $f^0(\cdot)=\phi(\cdot)^\top \theta^0$ for some 
bounded function $\phi(\cdot)$ and some 
$\theta^0\in \Theta$, with 
$\Theta \subseteq \R^p$ being compact, and assume that the identifiability condition~\eqref{eq:moment2c} holds.
Consider the function class 
$\mathcal{F} := \{f(\cdot) = \phi(\cdot)^\top \theta\,|\,
\theta \in \Theta\}$ and fixed bounded, continuously differentiable, characteristic kernels $k_Z$ and $k_R$ with bounded derivatives. 
Then, the estimator $\hat{f}$ defined in~\eqref{eq:hsic_objective} is consistent, i.e., 
$
\|\hat{f} - f^0\|_{\infty} \overset{\P_{M^0}}{\longrightarrow} 0$. 

The same statement holds if in~\eqref{eq:hsic_objective} we replace $\HSIC$ and $\widehat{\HSIC}$ by any independence measure $H(\mathbb{P}_{M^0},\theta)$ and its estimate
$\hat{H}_n(\mathcal{D}_n, \theta)$ (based on data $\mathcal{D}_n$), such that (i) $H(\mathbb{P}_{M^0},\theta) = 0$ if and only if $(Y - \phi(X)^\top\theta) \independent Z$, 
(ii) for all $\theta \in \Theta$, we have $\hat{H}_n(\mathcal{D}_n, \theta) \rightarrow H(\mathbb{P}_{M^0},\theta)$ in probability 
and, (iii), both 
$H(\mathbb{P}_{M^0},\theta)$ and 
$\hat{H}_n(\mathcal{D}_n, \theta)$ for all $n$ are Lipschitz continuous in $\theta$ with the same constant $L$.
\label{thm:consistency}
\end{theorem}

\subsection{Confidence Regions} \label{sec:confint}
Suppose 
$T_n: \Theta \times \mathbb{R}^{r\times n} \rightarrow \{0,1\}$ tests,  for a given parameter $\theta \in \Theta$, for independence between the residuals $R^\theta := Y - f_\theta(X)$ and the instruments $Z$, based on the $n$
i.i.d.\  observations $(X_i,Y_i,Z_i)_{i=1}^n$. 
Denote by $\Theta_0 \coloneqq \{\theta \,|\, R^\theta \independent Z\}$ the set of $\theta$'s satisfying the null hypothesis of independence.
We say $T_n$ has \emph{pointwise asymptotic level $\alpha$} if 
 $\sup_{\theta \in \Theta_0} \lim_{n \rightarrow \infty} \P(T_n(\theta, (Z_i)_{i=1}^n) = 1) \leq \alpha$
(here, $T_n=1$ corresponds to rejecting independence).
The test has \emph{uniform asymptotic level} if 
`lim' and `sup'
in the definition
can be exchanged and the statement still holds; it has \emph{finite sample level} if
it holds for all $n$ when removing 
`lim'.
If $T_n$ has finite sample level, then, by construction,
\begin{equation*}
\hat C_n := \{\theta \in \Theta \,|\, T_n(\theta, (Z_i)_{i=1}^n) = 0\}
\end{equation*}
is a $(1-\alpha)$-confidence region for $\theta^0$: %
$\P(\theta^0\in C_n ) = \P(T_n(\theta^0, (Z_i)_{i=1}^n) = 0) \geq 1 - \alpha$. Similarly for pointwise (uniform) asymptotic level. 

Because the independence restriction is stronger than the moment restriction (see Proposition~\ref{prop:stronger_identifability}), 
we can use any existing test for~\eqref{eq:moment1}, such as the Anderson-Rubin test \citep{AndersonRubin}, to construct such confidence regions, too.
Alternatively, we can use $\HSIC$ to construct an independence test, which we then invert. 
For a fixed kernel, tests can be constructed either by permutation-based procedures or by approximating 
the distributon of $\widehat{\HSIC}$ under the null hypothesis,
e.g., by a gamma distribution \citep[][]{Gretton2008}.
For many tests, $\hat{C_n}$ has to be approximated.

\section{Distribution Generalization} \label{sec:generalization}

Here, we follow a line of work, which 
connects distribution generalization with causality (see Section~\ref{sec:intro}). In this framework, distributional shifts are modeled as interventions. 
We propose using HSIC-X-pen, introduced in Section~\ref{sec:hsic_pen}, for this task
and prove that it is worst-case optimal under interventions on the exogenous variables, even if the model is nonlinear and the causal function is not identifiable.

To formalize the result, we again assume that the data is generated by the IV model $M^0$. Furthermore, in this section, we consider the function classes $\mathcal{F}\coloneqq \mathcal{L}^2(\R^d,\P_{M^0}^{X})$ consisting of all functions $f_{\diamond}:\R^d\rightarrow\R$ that satisfy $\E_{M^0}[f_{\diamond}(X)^2]<\infty$ and
\begin{equation*}
    \mathcal{F}_{\text{inv}}\coloneqq\left\{f_{\diamond}\in \mathcal{F}\mid Z\ci Y-f_{\diamond}(X) \text{ under } \P_{M^0}\right\}.
\end{equation*}
We model a distributional shift as %
an intervention on the exogenous $Z$, denoted by $i$; it consists of replacing the distribution of $Z$ with a new distribution.
The intervened model is denoted by $M^0(i)$ and is again an IV model. In this setting, we now prove that for any predictor $f\in\mathcal{F}$ satisfying the independence restriction $Z\ci Y-f(X)$, that is, $f\in\mathcal{F}_{\text{inv}}$, the expected loss is invariant to interventions on $Z$ in the following sense.
\begin{theorem}[Invariance with respect to interventions on $Z$]
\label{thm:invariant}
Let $\ell:\R\rightarrow\R$ be a convex loss function and $\mathcal{I}$ be a set of interventions on $Z$ satisfying
for all $i\in\mathcal{I}$ that $\P_{M^0(i)}$ is dominated by $\P_{M^0}$. Then, for all $f\in\mathcal{F}_{\text{inv}}$ it holds that
\begin{equation*}
    \E_{M^0}\big[\ell(Y-f(X))\big]=\sup_{i\in\mathcal{I}}\E_{M^0(i)}\big[\ell(Y-f(X))\big].
\end{equation*}
\end{theorem}
We consider this setting relevant as identifiablity (that is, $|\mathcal{F}_{\text{inv}}| = 1$) is not achievable in most modern machine learning applications.
The assumption that the intervened distributions are dominated by the observed distribution ensures that none of the interventions on $Z$ 
extend the support of $Z$. 
Generalizing to distributions that extend the support is only possible under additional extrapolation assumptions that require the functions $g^0$ and $f^0$ in the IV model $M^0$ to be partially identifiable \citep{Christiansen2020DG}. 
In this sense it is not possible to strengthen our result without making additional assumptions on the identifiablity of the IV model. 

Theorem~\ref{thm:invariant} motivates estimating a predictor based on
\begin{equation}
\label{eq:constrained_optim}
    \argmin_{f\in\mathcal{F}_{\text{inv}}}\, \E_{M^0}\big[\ell(Y-f(X))\big],
\end{equation}
where $\ell:\R\rightarrow\R$ is a convex loss function. Our proposed HSIC-X-pen estimator from \eqref{eq:penalized_HSIC_IV} for $\lambda$ approaching $0$ provides a flexible way of estimating the minimizer \eqref{eq:constrained_optim}. By Theorem~\ref{thm:invariant}, it is guaranteed to control the test error under any distributional shift generated by an intervention on $Z$ which does not extend the support. 
Moreover, 
we now prove
that for a sufficiently rich class of distribution shifts, a predictor solving \eqref{eq:constrained_optim} on the training data is worst-case optimal. To formalize our result, define the \emph{intervened subset} $S \subseteq\{1,\ldots,d\}$ consisting of all $X^j$ 
that are descendants of
$Z$ in the causal graph. When $|S|<d$, this may contain cases in which identifiability according to \eqref{eq:moment2} is impossible.\footnote{For example, if $Y=f^0(X_1, X_2)+U+\epsilon_Y$, $X_1=U$ and $X_2=Z+U$, then the causal function is not identifiable from \eqref{eq:moment2}. 
}
\begin{theorem}[Generalization to interventions on $Z$]
\label{thm:generalization}
Let $\ell:\R\rightarrow\R$ be a convex loss function and $\mathcal{I}$ be a set of interventions on $Z$ satisfying
for all $i\in\mathcal{I}$ that $\P_{M^0(i)}$ is dominated by $\P_{M^0}$, which itself is absolutely continuous with respect to a product measure. If there exists $i_*\in\mathcal{I}$ such that $X^S\ci U\mid X^{S^c}$ under $\P_{M^0(i_*)}$ and $\supp(\P_{M^0(i_*)}^{X})=\supp(\P_{M^0}^{X})$, then
 $$ \inf_{f\in\mathcal{F}_{\text{inv}\hspace{-0.5em}}}\E_{M^0\hspace{-0.2em}}\big[\ell(Y-f(X))\big]
     =\inf_{f\in \mathcal{F}}\sup_{i\in\mathcal{I}}\E_{M^0(i)\hspace{-0.3em}}\big[\ell(Y-f(X))\big].
    $$
\end{theorem}
The intervention $i_*$ can be called partially confounding-removing \citep[see also][]{Christiansen2020DG} in the sense that in the generated distribution $\P_{M^0(i_*)}$ the variables $X^S$ affected by the exogenous $Z$ are, conditioned on $X^{S^c}$, no longer confounded with $Y$ via $U$ (if $Z$ acts additively, this implies that, in general, $U$ does not act on both $Y$ and $X^S$).
As shown in the proof, this intervention results in the worst-case loss. An example of this type of intervention is given in Section~\ref{sec:sim_DG}. If $\ell(\cdot)=(\cdot)^2$ is the squared loss, the minimizer is attained at the conditional mean of $Y$ given $X$ under $\P_{M^0(i_*)}$, that is,
$f^0(x)+\E_{M^0(i_*)}\big[h(U,\epsilon^Y)\mid X=x\big]$ (see 
proof of Theorem~\ref{thm:generalization}).

\section{Experiments} \label{sec:exp}

\subsection{Simulation: Instrumental Variable Estimation}\label{sec:sim_exp}
We first evaluate the empirical performance of HSIC-X and HSIC-X-pen for estimating causal functions. To this end,
we use the following IV models in our experiments:
    \begin{equation}
        M(\alpha, \P_{\epsilon_Z}, f^0): \begin{cases}
        Z \coloneqq  \epsilon_{Z} \quad 
        U \coloneqq  \epsilon_{U} \\
        X \coloneqq  Z \epsilon_{X} + \alpha Z + U \\
        Y \coloneqq  f^0(X) - 4 U + \epsilon_Y,
        \end{cases} \label{eq:one_dim_IV_model}
\end{equation}
where $\epsilon_U, \epsilon_X, \epsilon_Y \simiid \mathcal{N}(0, 1)$, $\epsilon_Z \sim \P_{\epsilon_Z}$ 
and $\epsilon_Z \ci \epsilon_U, \epsilon_X, \epsilon_Y$. We consider different experiment settings by varying three parameters $\alpha, \P_{\epsilon_Z}$ and $f^0$. The parameter $\alpha$ adjusts how the instruments $Z$ influence the predictor $X$: the instruments $Z$ only change the variance of $X$ when $\alpha = 0$, while both the mean and the variance of $X$ are changed by $Z$ when $\alpha > 0$. 
For $\P_{\epsilon_Z}$,
we consider binary and Gaussian random variables. Lastly, 
for $f^0$,
we consider a linear function $f^0_{\text{lin}}(X) \coloneqq -2 X$ and a nonlinear function $f^0_{\text{nonlin}}(X) \coloneqq 1.5X - 0.2X^2 + \sum_{j=1}^{10} w_j e^{-(X - c_j)^2}$, where $c_j$ represents a partition of [-7, 7] in 10 equally spaced intervals and $w_1,\dots,w_{10} \simiid \mathcal{N}(0,2)$. We generate $1000$ observations from the IV model \eqref{eq:one_dim_IV_model} and evaluate the performance of our methods 
under different settings. In all the experiments, we use Adam as the optimizer with the learning rate set to $0.01$ and a batch-size of $256$. %

\begin{figure}[t]
\centerline{
\includegraphics[width=0.48\textwidth]{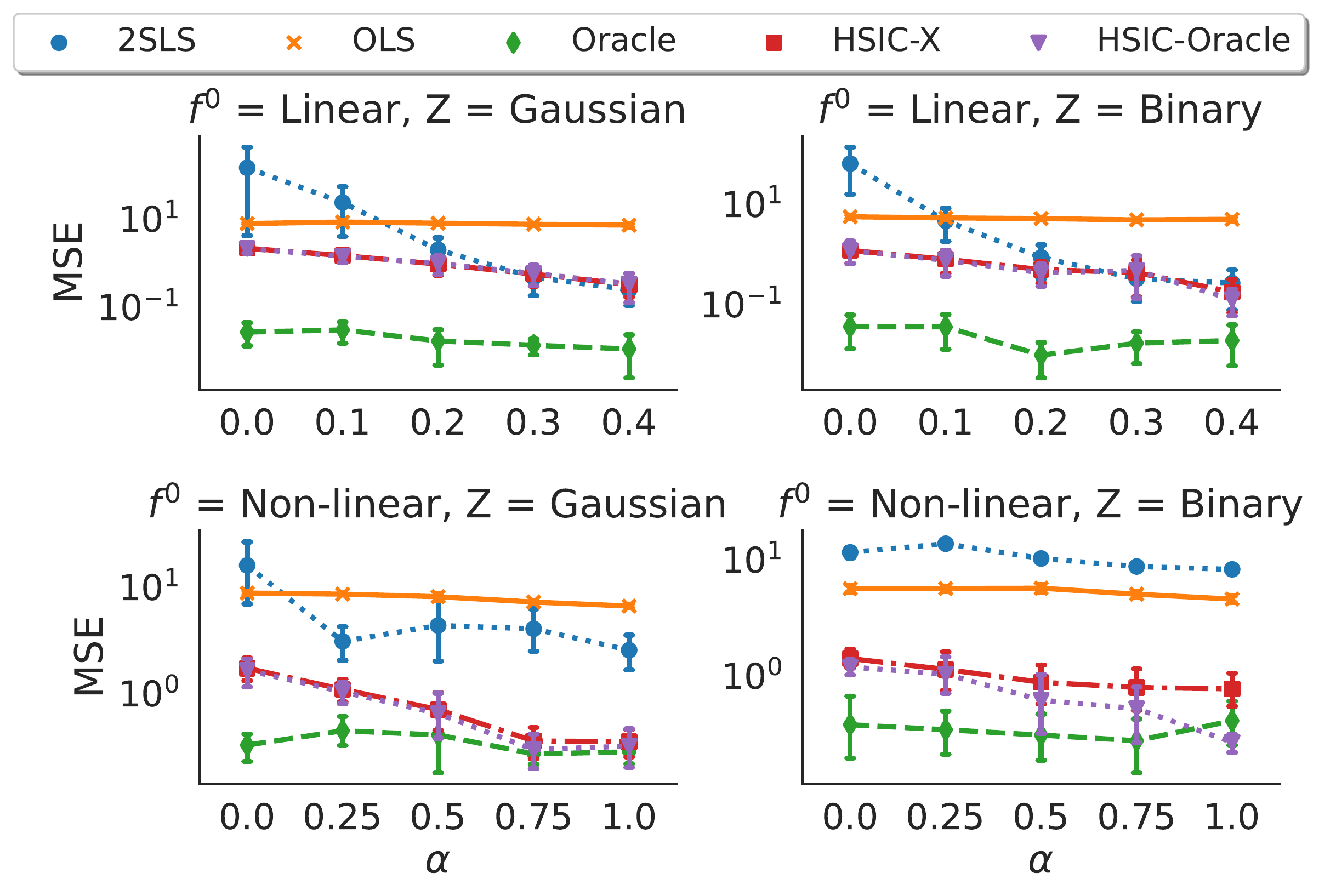}}
\caption{MSEs of different estimators when the correct basis functions are used. Each point represents an average over 10 simulations and the error bar indicates its 95\% confident interval. 2SLS is inconsistent and underperforms OLS when $\alpha = 0$, and for all $\alpha$ in the bottom right setting, while HSIC-X yields a substantial improvement over OLS in such settings. Under weak instruments (small $\alpha$), we observe an efficiency gain of HSIC-X over 2SLS. 
}
\label{fig:one_dim_basis}
\end{figure}

\paragraph{Known basis functions.}\label{sec:sim_iv_basis}
We first assume that the underlying function class containing 
the causal function is known; that is, we consider $\mathcal{F} \coloneqq \{ f(\cdot) = \phi(\cdot)^\top \theta \mid \theta \in \R^p \}$ with $p = 1$ and $\phi(x) = x$ in the linear case, and $p = 12$ and $\phi(x) = [x, x^2, e^{-(x-c_1)^2}, \dots, e^{-(x-c_{10})^2}]$ in the nonlinear case ($c_1,\dots,c_{10}$ are the same as described above). We compare the performance of HSIC-X
with the following baseline methods:
2SLS: this solves the moment restriction~\eqref{eq:moment1}; we use the feature map $\phi$ defined above as the function $\eta$ in the moment restriction.
OLS: least square regression of $Y$ on $X$ (using the basis functions $\phi$ or neural networks).
Oracle: least square regression of $Y$ on $X$ using non-confounded data (the confounder $U$ is removed from the assignment of $X$). 
HSIC-Oracle:
same as HSIC-X(-pen) but initialized at the true causal parameters; this is to investigate how much our method suffers from the non-convexity of the objective function. 
The last two methods 
serve as oracle benchmarks.

The performance of each method is measured by the integrated mean squared error (MSE) $\E[(\hat{f}(X) - f^0(X))^2]$ between the estimate $\hat{f}$ and the  causal function $f^0$, approximated using a test sample of size $10000$.
In all experiments, we report an average of the MSE values over 10 simulations. 

Figure~\ref{fig:one_dim_basis} reports the MSE as we increase the parameter $\alpha$. Our method shows a significant improvement over 2SLS in almost all settings. The improvement is especially prominent in the nonlinear-binary setting (see bottom right Figure~\ref{fig:one_dim_basis}), where 
the moment identifiability condition \eqref{eq:moment1b} is not satisfied (see Proposition~\ref{prop:stronger_identifability}). 
We still see an improvement gain even in the identifable cases (when $\alpha > 0$) which suggests a finite sample efficiency gain from using the independence restriction. Lastly, the performance of HSIC-X is on par with that of HSIC-Oracle, indicating that the optimization objective is reasonably well-behaved despite its non-convexity (only in the nonlinear/binary case, there is a slight deviation). We illustrate the estimated functions against the true causal function in Appendix~\ref{sec:app_exp_basis}.

\paragraph{Approximate functions.}
We now consider a more flexible function class to approximate the causal function $f^0_{\text{nonlin}}$ by using a neural network (NN).  For a fixed width and depth, $f^0_{\text{nonlin}}$ may not lie in the function class represented by the NN. Nonetheless, we expect our method to produce a reasonable estimate of the causal function.

We generate $1000$ observations from the IV model \eqref{eq:one_dim_IV_model} with 
$f^0_{\text{nonlin}}$. In addition to the OLS and Oracle baselines, we compare our method to DeepGMM \cite{Bennett2019} and DeepIV \cite{Hartford2017} using their publicly available implementations. To investigate the effect of the MSE regularization (see Section~\ref{sec:HSIC-X-Pen}), we add a variant of our method (HSIC-X-pen) where the MSE regularization is employed. A neural network with one hidden layer of size $64$ is used in our methods and all the baselines.

Figure~\ref{fig:exp_NN} outlines the simulation results. In short, both HSIC-X and HSIC-X-pen outperform DeepGMM, DeepIV and OLS baselines and approach the Oracle's performance as the instrument strength ($\alpha$) increases. We speculate that the inferior performance of DeepGMM and DeepIV may be 
explained by their sole reliance on the moment restriction \eqref{eq:moment1} as opposed to the full independence restriction \eqref{eq:moment2}. Lastly, HSIC-X-pen does not show a major improvement over HSIC-X.

\begin{figure}[t]
\centering
\includegraphics[width=0.48\textwidth]{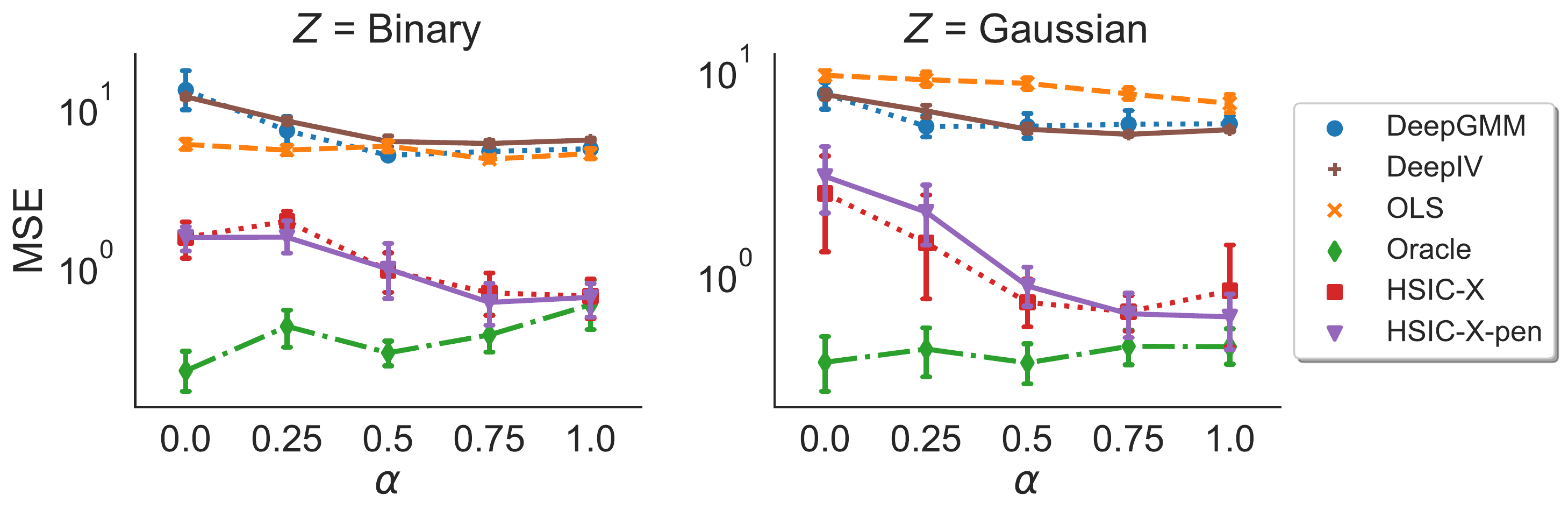}
\caption{MSEs of different estimators when the causal function is approximated by a neural network. HSIC-X shows a substantial gain over the baselines in all settings. HSIC-X-pen does not show a major improvement over HSIC-X.}
\label{fig:exp_NN}
\end{figure}

\paragraph{Multi-dimensional setting.}
We also investigate the effect of $X$'s and $Z$'s dimensionality on our
estimators. The exact experiment setup and detailed results are provided in Appendix~\ref{sec:app_exp_multidim}. The results suggest that for linear models with higher order effects of $Z$, HSIC-X outperforms 2SLS, in particular when the dimensions of $Z$ is strictly smaller than that of $X$.

\subsection{Simulation: Distribution Generalization}\label{sec:sim_DG}
To empirically verify the theoretical generalization guarantees from Section~\ref{sec:generalization}, we consider for $i\in\{0.5,1,\ldots,3.5,3.99\}$, the  collection of IV models
\begin{equation*}\label{eq:iv_model_DG}
    M(i): \begin{cases}
    Z \coloneqq  \epsilon_{i},  
    U_1 \coloneqq  \epsilon_{U_1},  
    U_2 \coloneqq  \epsilon_{U_2}, X_2 \coloneqq  U_2 + \epsilon_{X_2}\\
    X_1 \coloneqq  U_1 \mathds{1}(Z\leq 3.5) + 0.1 Z + 2Z\epsilon_{X_1}, \\
    Y \coloneqq  f^0(X_1, X_2) + U_1 + U_2,
    \end{cases}
\end{equation*}
where $\epsilon_{U_1},\epsilon_{U_2},\epsilon_{X_1}\epsilon_{X_2}$ are i.i.d.\ $\mathcal{N}(0, 1)$, $\epsilon_i=W_1\mathds{1}(K=0) + W_2\mathds{1}(K=1)$ with $K\sim\text{Ber}(i/4)$, $W_1\sim\text{Unif}(0,i)$ and $W_2\sim\text{Unif}(i, 4)$ and $\epsilon_i\ci ( \epsilon_{U_1},\epsilon_{U_2},\epsilon_{X_1}\epsilon_{X_2})$. The parameter $i$ determines how much the distribution of $Z$ with support $(0,4)$ is skewed towards the right. In particular, for $i\rightarrow 4$ the intervention removes the confounding effect of $U_1$, see Theorem~\ref{thm:generalization}. We use $M(0.5)$ to generate $3000$ observations from the training distribution and use it to fit HSIC-X-pen, OLS and Anchor regression (AR) \citep{Rothenhaeusler2018}, which comes with generalization guarantees, too,  but does not cover the higher-moments influence of $Z$ as in the SCM above. As in Section~\ref{sec:sim_exp}, we consider both linear and nonlinear functions for $f^0$, and employ the correct basis (Known Basis) and neural network (NN) function classes in our method and the baselines, see Appendix~\ref{sec:app_exp_DG} for details. We evaluate the fitted estimators on shifted test distributions and add a causal baseline (Causal) in which the true causal function is used as a predictor. This serves as an oracle baseline as, in this setting, the causal function is not identifiable.

The results are shown in Figure~\ref{fig:dis_gen}. Our method (HSIC-X-pen) and Causal are significantly more robust to the interventions compared to OLS and AR, with the causal oracle being the most invariant predictor as described in Theorem~\ref{thm:invariant}. However, the causal oracle is conservative (it ignores the information from the hidden confounder $U_2$ that is unaffected by the interventions and is helpful in predicting $Y$) and yields subpar performance when the interventions $i$ are small. Our method utilizes the invariant information from $U_2$ and yields the best trade-off among all candidates.
\begin{figure}[t]
\centering
\includegraphics[width=0.48\textwidth]{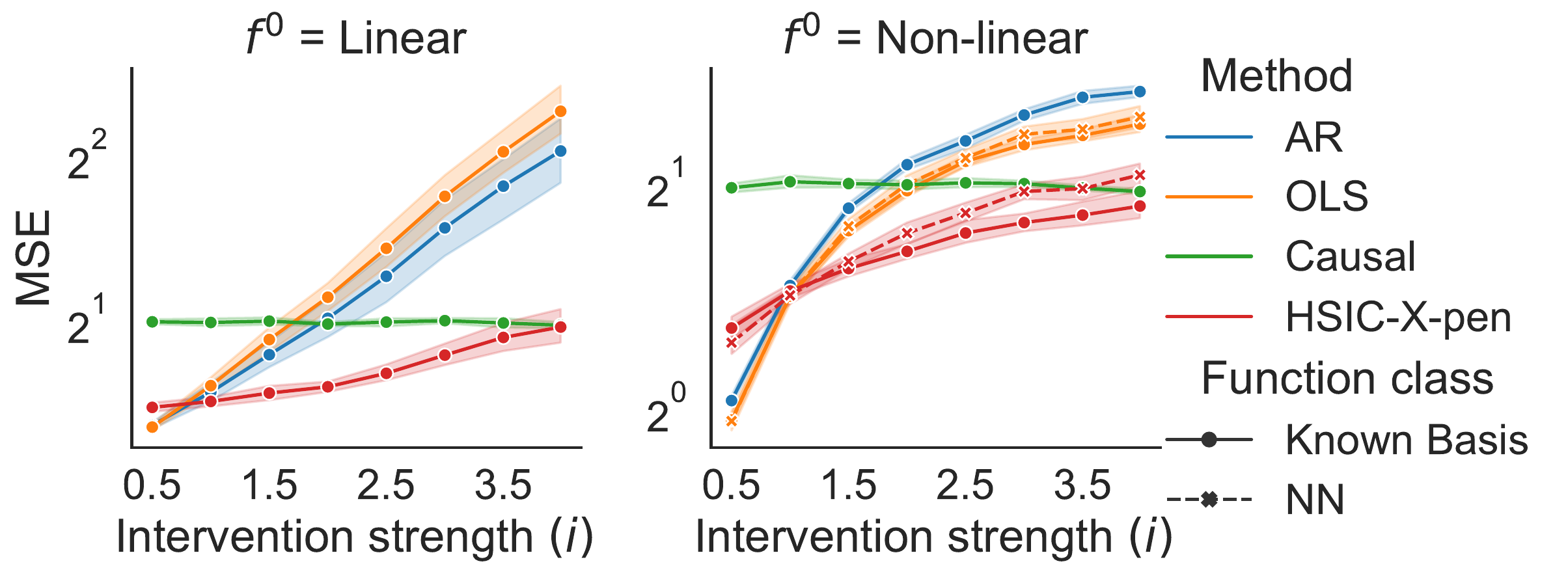}
\caption{Predictive performance of different predictors on shifted test distributions as the intervention strength increases. Causal is the most stable predictor but conservative, while OLS is markedly affected by the interventions. Due to the higher-moments influence from $Z$, HSIC-X outperforms AR in most cases and establishes the best trade-off between invariance and predictiveness.}
\label{fig:dis_gen}
\end{figure}

\subsection{Real-world Application}\label{sec:real_exp}
We apply HSIC-X to estimate the causal effect of education on earnings using 
$3010$ observations from the 1979 National Longitudinal Survey of Young Men
\citep{card1995}. The response variable $Y$ is the logarithm of wage, $X$ is years of education, and the (discrete) instrument $Z$ is geographic proximity to colleges (whether an individual grew up near a four-year college). \citet{card1995} also considered several conditioning variables $W$ including years of experience, race and geographic information
and studied a linear causal effect of education on earnings, so we use $\mathcal{F} = \{f \mid f(x) = \theta x\}$ and  apply the conditional IV estimator described in Appendix~\ref{sec:app_implmentation_civ}. 
In addition, we obtain a confidence region 
as described in Section~\ref{sec:confint}. We then compare our results with those by 2SLS, which was used in the original study of \citet{card1995}, and confidence intervals constructed by the Anderson-Rubin test \cite{AndersonRubin}.

Table~\ref{table:card} reports the point estimates and the confidence intervals at the 95\% confidence level. The point estimates from 2SLS and HSIC-X differ from the estimate from OLS suggesting a potential sizable unobserved confounding effect. Nonetheless, the difference between the OLS and 2SLS estimates is not statistically significant which was also observed in the original study (see \citet{card1995} and \citet{card2001}). We speculate that the linear effect of college proximity $Z$ on education $X$ may not be strong enough which leads to the imprecision of the 2SLS. On the other hand, the confidence region around HSIC-X does not contain the OLS. This suggests that taking into account full independence gains finite sample efficiency in this study.
\begin{table}[t]
\caption{Estimation results of the effect of education on earnings \citep{card1995}, including 95\% confidence sets. The independence restriction puts stronger constraints on the parameters and yields a smaller confidence set; e.g., it does not include the OLS solution, which  is not rejected by %
an Anderson-Rubin test.
}
\label{table:card}
\begin{center}
\begin{small}
\begin{sc}
\begin{tabular}{lcccr}
\toprule
Method & Point Estimate & Lower & Upper \\
\midrule
OLS    & 0.072 & 0.065 & 0.079 \\
2SLS & 0.142 & 0.050 & 0.273 \\
HSIC-X  & 0.160 & \textbf{0.097} & \textbf{0.208} \\
\bottomrule
\end{tabular}
\end{sc}
\end{small}
\end{center}
\vskip -0.1in
\end{table}

\section{Conclusion and Future Work}
Exploiting independence between exogenous variables and the residual terms of a response variable
can be beneficial for identifiability
of a causal function, the empirical performance of corresponding estimates,
and for constructing estimators that perform well even under interventions on such exogenous variables. 
We have proposed two estimators, HSIC-X and HSIC-X-pen that can be equipped with any 
machine learning regression method based on gradient descent. Empirical results on simulated and real data indicate that one may indeed benefit from considering independence restrictions in practice.

We believe it could be fruitful to 
construct fast approximate methods for inverting an independence test 
 and analyze distribution generalization when the support of $Z$ is extended.
 
\section*{Acknowledgments}
SS, LH, JP were supported by a research grant (18968) from VILLUM FONDEN. NP was supported by a research grant (0069071) from Novo Nordisk Fonden. We thank Nicola Gnecco and Nikolaj Thams for helpful discussions.

\bibliography{bibliographyJonas}
\bibliographystyle{icml2022}
\appendix 
\onecolumn

\section{Proofs} 
\label{sec:proofs}
\subsection{Proof of Proposition~\ref{prop:stronger_identifability}}
\begin{proof}
\begin{itemize}
    \item[(i)] 
We first show that if there exists $k$ and $\eta$, such that condition \eqref{eq:moment1b} holds, then condition \eqref{eq:moment2c} holds. We do so by contraposition, so assume that there exists a $\tau \neq 0$ violating condition \eqref{eq:moment2c}. The following then hold for all $k \in \mathbb{N}$ and all $\eta: \R^r \rightarrow \R^k$. First, $\Cov [ \eta(Z), h(U,\epsilon_Y) + \phi(X)^\top \tau ] = 0$. Second, as $Z \ci h(U,\epsilon_Y)$ it follows that $\Cov [ \eta(Z), h(U,\epsilon_Y)]=0$. Third, by combining the previous two results it follows that $\Cov [ \eta(Z), \phi(X)^\top \tau ] = 0$, which is a violation of condition \eqref{eq:moment1b}.

\item[(ii)] This statement is proven by (iii) (a), for example.
\item[(iii)] We now consider the two types of examples (a), (b) and (c).
\paragraph{(a)}

In the example SCM, we have $\E[X|Z]=Z \E[U] + \E[\epsilon_X]=0$ and therefore for all $k \in \mathbb{N}$ and functions $\eta: \R^r \rightarrow \R^k$, $\E[\eta(Z) X] = 0$.
As a result, the identifiability condition \eqref{eq:moment1b} does not hold here. Specifically, any function of the form $\theta X, \theta \in \R$ satisfies the moment restriction, so we cannot identify the causal function $f^0(x) = x$.

Consider now the independence based identifiability condition \eqref{eq:moment2c}:
\[
Z \ci U + \epsilon_Y + \tau X \, \implies \, \tau=0.
\]

Plugging in the structural equation for $X$ this is equivalent to $K_{\tau} \coloneqq U + \tau ZU + \tau \epsilon_X + \epsilon_Y \not\ci Z$ for all $\tau \neq 0$. We now show that this is true by showing that for all $\tau \neq 0, \mathbb{E}[Z^2K^2_{\tau}] \neq \mathbb{E}[Z^2]\mathbb{E}[K^2_{\tau}]$. First,
\begin{align*}
\mathbb{E}[Z^2K^2_{\tau}] 
& =\mathbb{E}[Z^2(U + \tau ZU + \tau \epsilon_X + \epsilon_Y)^2]\\
& =\mathbb{E}[Z^2U^2] + \tau^2 \mathbb{E}[Z^4U^2] + \tau^2 \mathbb{E}[Z^2\epsilon^2_X] + \mathbb{E}[Z^2\epsilon^2_Y]\\
& =2 + 4\tau^2. 
\end{align*}
Second,
\begin{align*}
\mathbb{E}[Z^2] \mathbb{E}[K^2_{\tau}] 
& =\mathbb{E}[Z^2] \mathbb{E}[(U + \tau ZU + \tau \epsilon_X + \epsilon_Y)^2]\\
& =\mathbb{E}[Z^2](\mathbb{E}[U^2] + \tau^2 \mathbb{E}[Z^2U^2] + \tau^2 \mathbb{E}[\epsilon_X^2] + \mathbb{E}[\epsilon_Y^2])\\
& =2 + 2 \tau^2. 
\end{align*}
Therefore, condition \eqref{eq:moment2c} is fulfilled, meaning we can identify the causal function $f^0(x)=x$.

    \paragraph{(b)} We first give an argument why identifiability using condition~\eqref{eq:moment1b} is impossible. For a discrete instrument $Z$, the number of almost surely linearly independent functions\footnote{Here, we say that a collection of functions $l_1,\dots,l_k$ is almost surely linearly independent if for all vectors $\alpha \neq 0$, $\P\left(\sum_{i=1}^k \alpha_i l_i(Z) = 0\right)=0$.} of $Z$ is bounded by the cardinality of the support of $Z$; therefore, if the number of basis functions is larger than the support of $Z$, there is no  $k\in\mathbb{N}$ and $\eta:\R^r\rightarrow\R^k$ such that condition~\eqref{eq:moment1b}  is satisfied.

We now construct an explicit example as follows: First, we show for a class of SCMs that under certain assumptions condition \eqref{eq:moment1b} does not hold while \eqref{eq:moment2c} does. Second, we give an explicit SCM with a binomially distributed instrument that fulfills these conditions.
Fix $k,p \in \mathbb{N}$ such that $p>k$ and consider the SCM
\begin{align*}
    Z&\coloneqq \epsilon_Z\\
    U&\coloneqq \epsilon_{U}\\
    X&\coloneqq Z  + U  + \epsilon_X \\
    Y&\coloneqq \sum_{i=1}^{p} \alpha_i X^i + U + \epsilon_Y,
  \end{align*}
  where $\epsilon_Z,\epsilon_U,\epsilon_X$ and $\epsilon_Y$ are jointly independent, real-valued errors $\epsilon_Z \sim \mathrm{Unif}(\{0,\dots,k-1\})$
  and $ \alpha_i \in \mathbb{R}\setminus \{0\}$ for all $i\in \{1,\dots,p\}$.
Assume that the random variables $\{\epsilon_X^iU^j | 1 \leq i+j \leq p\}$ are almost surely linearly independent (see above).
  We are interested in estimating the causal function $f^0(x)= \sum_{i=1}^{p} \alpha_i x^i$ with the basis functions $\phi(x)=(x^1,\dots,x^p)$.

  Consider first the moment identifiability condition \eqref{eq:moment1b}. Because $Z$ is discrete, the set of functions $\{\mathbf{1}_{\{Z=i\}}, i \in \{0,\dots,k-1\}\}$ form a linear basis for all functions from
  $\supp(Z)$ to $\mathbb{R}$. As a result, for all functions  $\eta= (\eta_1, \dots, \eta_q): \R \rightarrow \R^q$, with $q>k$, condition~\eqref{eq:moment1b} cannot hold. 
This implies that while condition \eqref{eq:moment1b} may hold if $p \leq k$, with for example $\eta(Z)=(\mathbf{1}_{\{Z=0\}},\dots,\mathbf{1}_{\{Z=k-1\}})$, it cannot hold if $p > k$, irrespective of the choice of $\eta$.
  
 Consider now the independence identifiability condition \eqref{eq:moment2c}:
\[
\sum_{i=1}^{p} \tau_i \alpha_i X^i + h(U, \epsilon_Y) \ci Z \quad \implies \tau_1=\dots=\tau_p=0.
\]  
Because the support of $Z$ is finite, this is equivalent to the statement that the distributions of the random variables
\[
\sum_{i=1}^{p} \tau_i \alpha_i (j + U + \epsilon_X)^i + U + \epsilon_Y,
\]
are the same for all $j \in \{ 0,\dots,k-1 \}$. Consider the cases $j=0$ and $j=1$. For $\sum_{i=1}^{p} \tau_i \alpha_i (U + \epsilon_X)^i + U + \epsilon_Y$ and 
$\sum_{i=1}^{p} \tau_i \alpha_i (1 + U + \epsilon_X)^i + U + \epsilon_Y$ to have the same distribution, the coefficients in front of the $\epsilon_X^i$ (when collecting all terms)  for all $i\in \{1,\dots,p\}$ must be equal 
in both random variables. Consider first the $\epsilon_X^{p-1}$ term. The corresponding two coefficients are $\tau_{p-1} \alpha_{p-1}$ and $\tau_{p-1} \alpha_{p-1} + p \tau_{p} \alpha_{p}$. These two coefficients are only equal if $\tau_{p}=0$. By iterating this argument we obtain that $\tau_2=\dots=\tau_p=0$. 
This implies that $\tau_{1} \alpha_{1}(U + \epsilon_X) + U + \epsilon_Y$ and $\tau_{1} \alpha_{1}(1 + U + \epsilon_X) + U + \epsilon_Y$ have the same distribution. Thus, $\tau_1=0$ 
and therefore condition~\eqref{eq:moment2c} holds.

We now give an explicit example, where $k<p$ and where the random variables $\{\epsilon_X^iU^j | 1 \leq i+j \leq p\}$ are almost surely linearly independent. Let $\epsilon_U,\epsilon_Z,\epsilon_X$ and $\epsilon_Y$  standard normal, $Z\sim \mathrm{Bernoulli}(0.5)$ and $p=3$. Here, $k=2 < p$ and we will now show that the set of random variables $\{\epsilon_X^iU^j|1\leq i+j\leq 3\}$ is almost sure linearly independent. 
Let $\tau_{ij},1\leq i+j \leq 3$
be coefficients
such that 
$K_{\tau} \coloneqq \sum_{1 \leq i+j \leq 3} \tau_{ij} \epsilon_X^iU^j = 0$ almost surely. In particular, this implies that $\EX[K_{\tau} \epsilon_X^iU^j]=0$ for all $i,j \leq 0$. 
Using 
$\mathbb{E}[\epsilon_X^iU^j \epsilon_X^kU^l]=0$ 
if either $i+k$ or $j+l$ are odd
(which follows as $\epsilon_X$ and $U$ are mean zero Gaussian and independent),
we obtain the following nine equations.
\begin{align*}
    &\mathbb{E}[K_{\tau}U]=
    \tau_{01}+3\tau_{03}+\tau_{21} =0\\
    &\mathbb{E}[K_{\tau}U^2]=
    3\tau_{02}+\tau_{20} =0 \\
    &\mathbb{E}[K_{\tau}U^3]=
    3\tau_{01}+15\tau_{03}+3\tau_{21} =0 \\
    &\mathbb{E}[K_{\tau}\epsilon_X]=
    \tau_{10} + \tau_{12} + 3\tau_{30} =0\\
    &\mathbb{E}[K_{\tau}\epsilon_XU]=
    \tau_{11} =0\\
    &\mathbb{E}[K_{\tau}\epsilon_XU^2]=
    \tau_{10} + 3\tau_{12} + 3\tau_{30} =0\\
    &\mathbb{E}[K_{\tau}\epsilon_X^2]=
    \tau_{02} + 3 \tau_{20} =0\\
    &\mathbb{E}[K_{\tau}\epsilon_X^2U]=
    \tau_{01} + 3\tau_{03} + 3\tau_{21} =0\\
    &\mathbb{E}[K_{\tau}\epsilon_X^3]=
   3\tau_{10} + 3\tau_{12} + 15\tau_{30} =0. 
\end{align*}
We can write this as a linear system 
$A\tau=0$
with corresponding $9\cross 9$ matrix $A$. As $A$ is invertible, $\tau = 0$ follows.

\paragraph{(c)}
Consider the SCM
\begin{align*}
    &Z := \epsilon_Z \\
    &U := \epsilon_U\\
    &X := 2ZU - U + \epsilon_X\\
    &Y := X + U + \epsilon_Y.
\end{align*}
where $\epsilon_Z,\epsilon_U,\epsilon_X$  and $\epsilon_Y$ are jointly independent, $\epsilon_Z \sim  \mathrm{Bernoulli}(0.5)$ and the remaining errors are standard normal. Consider the basis function $\phi(x)=x$. Here,
\begin{eqnarray*}
K_{\tau} \coloneqq h(U,\epsilon_Y) + \tau X=
     \begin{cases}
     (1+\tau) U +  \tau \epsilon_X + \epsilon_Y & \quad \quad \textrm{if }Z=1  \\
     (1-\tau) U +  \tau \epsilon_X + \epsilon_Y & \quad \quad \textrm{if }Z=0  \\
     \end{cases}\
\end{eqnarray*}
and therefore $\mathbb{E}[K^2_{\tau}|Z=1]=2\tau^2+2\tau+2$ and $\mathbb{E}[K^2_{\tau}|Z=0]=2\tau^2-\tau+2$. It follows that condition~\eqref{eq:moment2c} holds. On the other hand, $\tau X|(Z=1)=\tau(U + \epsilon_X)$ and $\tau X|(Z=0)=\tau(-U + \epsilon_X)$ and second, the distribution of $U$ is symmetric around 0. Therefore, $Z \ci \tau X$ for all $\tau \in \mathbb{R}$. We can therefore identify the causal function with the independence restriction, even though $Z$ is independent of $X$.

 \end{itemize}

\end{proof}

\subsection{Proof of Proposition~\ref{prop:no-simplification}}
\begin{proof}
This is a proof by example. We have already constructed an SCM, such that condition~\eqref{eq:moment2c} holds but condition~\eqref{eq:moment2d} does not in the proof of statement (iii) (c) of Proposition \ref{prop:stronger_identifability}.

We now construct a SCM, such that 
condition~\eqref{eq:moment2d} holds but 
condition~\eqref{eq:moment2c} does not. 
Consider the SCM
\begin{align*}
    &Z := \epsilon_Z\\
    &U := \epsilon_U\\
    &X := 2ZU + \epsilon_X\\
    &Y := X - U  + \epsilon_Y.
\end{align*}
where $\epsilon_Z,\epsilon_U,\epsilon_X$  and $\epsilon_Y$ are jointly independent, $\epsilon_Z \sim  \mathrm{Bernoulli}(0.5)$ and the remaining errors are standard normal. Consider the basis functions $\phi(x)=(x)$. Here, 
\begin{eqnarray*}
h(U,\epsilon_Y) + X=
     \begin{cases}
     U + \epsilon_X + \epsilon_Y & \quad \quad \textrm{if }Z=1  \\
     - U + \epsilon_X + \epsilon_Y & \quad \quad \textrm{if }Z=0  \\
     \end{cases}
\end{eqnarray*}
and therefore condition~\eqref{eq:moment2c} does not hold for $\tau=1$. On the other hand,
\begin{eqnarray*}
\tau X =
     \begin{cases}
     \tau (2U + \epsilon_X) & \quad \quad \textrm{if }Z=1  \\
     \tau \epsilon_X & \quad \quad \textrm{if }Z=0  \\
     \end{cases}
\end{eqnarray*}
and therefore $\mathbb{E}[(\tau X)^2|Z=1]=5\tau^2$ while $\mathbb{E}[(\tau X)^2|Z=0]=\tau^2$. It follows that condition~\eqref{eq:moment2d} holds. 
\end{proof}

\subsection{Proof of Theorem~\ref{thm:consistency}}
\begin{proof}
Let $H(\mathbb{P}_{M^0},\theta)$ $\coloneqq \mathrm{HSIC}((Y -\phi(X)^\top \theta,Z);k_{R},k_{Z})$ with non-negative bounded and Lipschitz continuous kernels $k_R$ and $k_Z$ with bounded and continuous derivatives.\footnote{From now on, we do not explicitly write down the kernels and write $H(\theta)$ rather than $H(\mathbb{P}_{M^0},\theta)$ to ease notation} Consider an i.i.d.\ data set $\mathcal{D}_n=(X_i,Y_i,Z_i)_{i=1}^n$, such that that each triple $(X_i,Y_i,Z_i)$ follows the distribution of $M^0$. Let $\hat{H}_{n}(\mathcal{D}_n,\theta)\coloneqq\widehat{\HSIC}((R_i^{\theta}, Z_i)_{i=1}^n)$, with $R_i^{\theta}=Y_i - \phi(X_i)^\top \theta$.

Then the following three results hold. First, as the kernels $k_R$ and $k_L$ are by assumption non-negative and bounded it follows by Corollary 15 in \citet{Mooij2016jmlr} that for all $\theta$ and all $\epsilon>0$, $\lim_{n\to \infty} \mathbb{P}_{M^0}(|\hat{H}_{n}(\mathcal{D}_n,\theta) - H(\theta)|>\epsilon)=0$.

Second, it follows by Lemma 16 in \citet{Mooij2016jmlr} that for all $\theta_1,\theta_2 \in \Theta$ and all $n \geq 2$ 
\begin{align*}
\hat{H}_n(\mathcal{D}_n,\theta_1) - \hat{H}_n(\mathcal{D}_n,\theta_2)
&= \widehat{\HSIC}((R_i^{\theta_1}, Z_i)_{i=1}^n)-\widehat{\HSIC}((R_i^{\theta_2}, Z_i)_{i=1}^n) \\
&\leq \frac{32\lambda C}{\sqrt{n}}\norm{R^{\theta_1}-R^{\theta_2}} \\
&= \frac{32\lambda C}{\sqrt{n}}\norm{\phi(X)^\top (\theta_1 - \theta_2)} \\
&\leq \frac{32\lambda C}{\sqrt{n}}\norm{\phi(X)^\top} \norm{(\theta_1 - \theta_2)} \\
&=  \frac{32\lambda C}{\sqrt{n}} \left(\sum_{i=1}^n \norm{\phi(X_i)}^2\right)^{0.5} \norm{(\theta_1 - \theta_2)}\\
&\leq 32\lambda C M \norm{(\theta_1 - \theta_2)},
\end{align*}
where $R^\theta=(Y_i - \phi(X_i)^\top \theta)_{i=1}^n$, $\phi(X)=(\phi(X_i))_{i=1}^n$, $\lambda$ is the Lipschitz constant for $k_R$, $C$ is an upper bound for $k_Z$ and $M$ is an upper bound for $\norm{\phi(\cdot)}^2$ 
which exists by the assumption that $\phi(\cdot)$ is bounded. 

Third, for example by \citet{Pfister2017jrssb} Proposition~2.5, the population HSIC can be written as 
\begin{align*}
    H(\theta) 
    &= \mathrm{HSIC}(Y -\phi(X)^\top \theta,Z) \\
    &= \mathbb{E}[k_R(R_1^{\theta_1},R_2^{\theta_2})k_Z(Z_1,Z_2)]  
    + \mathbb{E}[k_R(R_1^{\theta_1},R_2^{\theta_2})] \mathbb{E}[k_Z(Z_1,Z_2)]  
    -2 \mathbb{E}[k_R(R_1^{\theta_1},R_2^{\theta_2})k_Z(Z_1,Z_3)],
\end{align*}
with $R_1^{\theta_1}$ and $R_2^{\theta_2}$ as well as $Z_1$, $Z_2$ and $Z_3$ being i.i.d.\ copies of $R^{\theta_1}$ and $Z$, respectively.
The function $R^\theta : \theta \mapsto Y-\phi(X)^\top \theta$ is (surely) continuously differentiable, as is $k_R$ by assumption. Further, $k_R$ and $k_Z$ have bounded derivative. We can therefore conclude by dominated convergence and the chain rule that $H(\theta)$ is continuously differentiable. Therefore, $H(\theta)$ is Lipschitz on $\Theta$, since $\Theta$ is compact.

Jointly, these three results (which correspond to (i), (ii), and (iii) from the statement of the theorem) allow us to apply Corollary 2.2 from \citet{newey1991uniform} and conclude that for all $\varepsilon >0$,
\begin{equation}
    \label{eq:uniform_consistency}
    \lim_{n\to \infty} \mathbb{P}_{M^0}(\max_{\theta \in \Theta}|\hat{H}_n(\mathcal{D}_n,\theta) - H(\theta)|>\varepsilon)=0.
\end{equation}
Here and in the remainder of the proof we use that $\theta \mapsto \hat{H}_n(\mathcal{D}_n,\theta)$ and $\theta \mapsto H(\theta)$ are (surely) continuous functions that map to $\R$. Restricted to any compact set, they therefore (surely) attain a maximum and a minimum.

Consider now any sequence of estimators $\hat{\theta}_{n}^0 \in \mathrm{arg}\,\min_{\theta \in \Theta} \hat{H}_{n}(\mathcal{D}_n,\theta)$. By the assumption that condition \eqref{eq:moment2c} holds, $\theta^0$ is the unique minimizer of the continuous function $H(\cdot)$ on the compact set $\Theta$. Therefore, for all $\varepsilon>0$, there exists a $\zeta(\varepsilon)$ such that for all $\theta \notin B_{\varepsilon}(\theta^0)=\{\theta:||\theta-\theta^0|| < \varepsilon\}$ it holds that $H(\theta)-H(\theta^0) > \zeta(\varepsilon)$. Note also that $\Theta \setminus B_{\varepsilon}$ is compact.
Fix $\varepsilon>0$ and $\delta>0$, then by \eqref{eq:uniform_consistency} there exists an $n'\in\mathbb{N}$, such that for all $n \geq n'$
\begin{equation}
\label{eq:upperbound}
    \mathbb{P}_{M^0}(\mathrm{max}_{\theta \in \Theta} |\hat{H}_{n}(\mathcal{D}_n,\theta)-H(\theta)|> \frac{1}{2}\zeta(\varepsilon)) < \frac{1}{2}\delta.
\end{equation}
Then, for all $n\geq n'$ it holds that
\begin{align*}
\mathbb{P}_{M^0}(||\hat{\theta}_n^0 - \theta^0|| > \varepsilon) 
&\leq \mathbb{P}_{M^0}(\mathrm{min}_{\theta \in \Theta \setminus B_{\varepsilon}(\theta^0)}( \hat{H}_{n}(\mathcal{D}_n,\theta)- \hat{H}_n(\mathcal{D}_n,\theta^0)) \leq 0) \\
&\leq \mathbb{P}_{M^0}(\mathrm{max}_{\theta \in \Theta \setminus B_{\varepsilon}(\theta^0)}|\hat{H}_{n}(\mathcal{D}_n,\theta) - H(\theta)| + |\hat{H}_n(\mathcal{D}_n,\theta^0)-H(\theta^0)|>\zeta(\varepsilon)) \\
&\leq \mathbb{P}_{M^0}(\{\mathrm{max}_{\theta \in \Theta \setminus B_{\varepsilon}(\theta^0)}|\hat{H}(\mathcal{D}_n,\theta) - H(\theta)| > \frac{1}{2}\zeta(\varepsilon)\}  \{|\hat{H}_n(\mathcal{D}_n,\theta^0)-H(\theta)|> \frac{1}{2}\zeta(\varepsilon)\})\\
&< \delta.
\end{align*}
Here, we used the following four arguments:
\begin{itemize}
    \item Firstly, the event $\{||\hat{\theta}_n^0 - \theta^0|| > \varepsilon\}$ can only occur if there exists a $\theta \in \Theta\setminus B_{\varepsilon}(\theta^0)$ such that the event $\{\hat{H}_{n}(\mathcal{D}_n,\theta)- \hat{H}_n(\mathcal{D}_n,\theta^0) \leq 0\}$ occurs.
    \item Secondly, $\zeta(\varepsilon)$ is defined such that for all $\theta \in \Theta\setminus B_{\varepsilon}(\theta^0)$ it holds that $H(\theta) - H(\theta^0) > \zeta(\varepsilon)$. Therefore, if there exist a $\theta$ such that the event $\{\hat{H}_{n}(\mathcal{D}_n,\theta)- \hat{H}_n(\mathcal{D}_n,\theta^0) \leq 0\}$ occurs, then there must exist a $\theta$ such that the event $\{|\hat{H}_{n}(\mathcal{D}_n,\theta) - H(\theta)| + |\hat{H}_n(\mathcal{D}_n,\theta^0)-H(\theta^0)|>\zeta(\varepsilon)\}$ occurs.
    \item Thirdly, for $\{\mathrm{max}_{\theta \in \Theta \setminus B_{\varepsilon}(\theta^0)}|\hat{H}_{n}(\mathcal{D}_n,\theta) - H(\theta)| + |\hat{H}_n(\mathcal{D}_n,\theta^0)-H(\theta^0)|>\zeta(\varepsilon)\}$ to occur, either $\{\mathrm{max}_{\theta \in \Theta \setminus B_{\varepsilon}(\theta^0)}|\hat{H}_{n}(\mathcal{D}_n,\theta) - H(\theta)| > \frac{1}{2} \zeta(\varepsilon)\}$ or $\{|\hat{H}_n(\mathcal{D}_n,\theta^0)-H(\theta^0)|>\frac{1}{2}\zeta(\varepsilon)\}$ must occur.
    \item Fourthly, we use the union bound together with \eqref{eq:upperbound} for the last inequality.
\end{itemize}
We can therefore conclude that $\lim_{n \to \infty} \mathbb{P}_{M^0}(||\hat{\theta}_{n}^0 - \theta^0||>\varepsilon)=0.$ 

Finally, consider our estimator $\hat{f}(\cdot)=\phi(\cdot)^\top \hat{\theta}_{n}^0$ for the causal function $f^0$. 
By the assumption that $\phi$ is bounded, we get that
\begin{equation*}
    ||\hat{f}-f^0||_{\infty}=\sup_{x\in\R^d}|\phi(x)^{\top}(\hat{\theta}^0_n-\theta^0)|\leq \sup_{x\in\R^d}||\phi(x)||\cdot ||\hat{\theta}^0_n-\theta^0||\leq M ||\hat{\theta}^0_n-\theta^0||.
\end{equation*}
Hence we conclude that $||\hat{f}-f^0||_{\infty} \xrightarrow{\mathbb{P}_{M^0}} 0$. 

Our second claim follows by the fact that under the conditions laid out, we can directly apply Corollary 2.2 from \citet{newey1991uniform} to obtain uniform convergence in probability and then argue as in the case for $\widehat{\HSIC}$.

\end{proof}

\subsection{Proof of Theorem~\ref{thm:invariant}}

\begin{proof}
Let $f_{\diamond}\in \mathcal{F}_{\text{inv}}$ and $i\in\mathcal{I}$. Let $\mathcal{G}$ be the directed graph induced by the SCM $M^0$. Then, since $Z$ is by construction a source node in $\mathcal{G}$ and $i$ an intervention on $Z$ satisfying that $\P_{M^0(i)}$ is dominated by $\P_{M^0}$, it holds that
\begin{align}
    \E_{M^0(i)}\big[\ell(Y-f_{\diamond}(X))\big]
    &=\E_{M^0(i)}\Big[\E_{M^0(i)}\big[\ell(Y-f_{\diamond}(X))  \mid Z\big]\Big] \nonumber \\
    &=\E_{M^0(i)}\Big[\E_{M^0}\big[\ell(Y-f_{\diamond}(X))\mid Z\big]\Big]. 
    \label{eq:i_does_not_matter1}
\end{align}
Now, since $f_{\diamond}\in\mathcal{F}_{\text{inv}}$ it holds that $Z\ci Y-f_{\diamond}(X)$ under $\P_{M^0}$, which implies that
\begin{equation}
    \label{eq:i_does_not_matter2}
    \E_{M^0(i)}\Big[\E_{M^0}\big[\ell(Y-f_{\diamond}(X))\mid Z\big]\Big]
    =\E_{M^0(i)}\Big[\E_{M^0}\big[\ell(Y-f_{\diamond}(X))\big]\Big]
    =\E_{M^0}\big[\ell(Y-f_{\diamond}(X))\big]\Big].
\end{equation}
Since $i\in\mathcal{I}$ was arbitrary, combining \eqref{eq:i_does_not_matter1} and \eqref{eq:i_does_not_matter2} directly implies
\begin{equation*}
    \E_{M^0}\big[\ell(Y-f(X))\big]=\sup_{i\in\mathcal{I}}\E_{M^0(i)}\big[\ell(Y-f(X))\big],
\end{equation*}
which completes the proof of Theorem~\ref{thm:invariant}.
\end{proof}

\subsection{Proof of Theorem~\ref{thm:generalization}}

\begin{proof}
We consider two optimization problems: (A) Minimize $\E_{M^0}\big[\ell(Y-f_{\diamond}(X))\big]$ over all $f_{\diamond}\in\mathcal{F}_{\text{inv}}$ and (B) minimize $\sup_{i\in\mathcal{I}}\E_{M^0(i)}\big[\ell(Y-f_{\diamond}(X))\big]$ over all $f_{\diamond}\in\mathcal{F}$. To prove the result, we fix an arbitrary $\varepsilon>0$ and find a function $\overline{\phi}\in\mathcal{F}_{\text{inv}}$ that satisfies
\begin{equation}
    \label{eq:eps_close_A}
    \E_{M^0}\big[\ell(Y-\overline{\phi}(X))\big]\leq \inf_{f_{\diamond}\in\mathcal{F_{\text{inv}}}}\E_{M^0}\big[\ell(Y-f_{\diamond}(X))\big] + \epsilon
\end{equation}
and
\begin{equation}
    \label{eq:eps_close_B}
    \sup_{i\in\mathcal{I}}\E_{M^0(i)}\big[\ell(Y-\overline{\phi}(X))\big]\leq \inf_{f_{\diamond}\in\mathcal{F}}\sup_{i\in\mathcal{I}}\E_{M^0(i)}\big[\ell(Y-f_{\diamond}(X))\big] + \epsilon.
\end{equation}
Since by \eqref{eq:i_does_not_matter1} and \eqref{eq:i_does_not_matter2} in the proof of Theorem~\ref{thm:invariant} it holds that $\E_{M^0}\big[\ell(Y-\overline{\phi}(X))\big]= \sup_{i\in\mathcal{I}}\E_{M^0(i)}\big[\ell(Y-\overline{\phi}(X))\big]$, the result we wish to prove follows immediately. The proof now proceeds in two steps:
\begin{enumerate}
    \item[(1)] We construct the function $\overline{\phi}$ using the intervention $i_*$ from the statement.
    \item[(2)] We use that $\overline{\phi}\in\mathcal{F}_{\inv}$ to conclude the proof.
\end{enumerate}

\textbf{Step (1):} Let $i_*\in\mathcal{I}$ be the intervention from the statement of the theorem. For all $f_{\diamond}\in \mathcal{F}$, it holds that
\begin{equation*}
    \sup_{i\in\mathcal{I}}\E_{M^0(i)}\big[\ell(Y-f_{\diamond}(X))\big]
    \geq \E_{M^0(i_*)}\big[\ell(Y-f_{\diamond}(X))\big].
\end{equation*}
As $f_{\diamond}\in\mathcal{F}$ was arbitrary, we can take the infimum on both sides and get
\begin{equation}
    \label{eq:minimum_part1}
    \inf_{f_{\diamond}\in \mathcal{F}}\sup_{i\in\mathcal{I}}\E_{M^0(i)}\big[\ell(Y-f_{\diamond}(X))\big]
    \geq \inf_{f_{\diamond}\in \mathcal{F}}\E_{M^0(i_*)}\big[\ell(Y-f_{\diamond}(X))\big].
\end{equation}
By the definition of the infimum there exists\footnote{In the case of squared loss $\ell(\cdot)=(\cdot)^2$, the minimum is attained at the conditional mean $\E_{M^{0}(i_*)}[Y\mid X]=f^0(X)+\E_{M^{0}(i_*)}[h(U,\epsilon_Y)\mid X]$.} $\psi\in\mathcal{F}$ satisfying
\begin{equation}
    \label{eq:minimum_part2}
    \E_{M^0(i_*)}\big[\ell(Y-\psi(X))\big]
    \leq\inf_{f_{\diamond}\in \mathcal{F}}\sup_{i\in\mathcal{I}}\E_{M^0(i)}\big[\ell(Y-f_{\diamond}(X))\big]+\varepsilon/2.
\end{equation}
Moreover, setting $\overline{\psi}\coloneqq\psi-f^0$ and expanding $Y$ from the structural assignment we get that
\begin{equation*}
    \E_{M^0(i_*)}\big[\ell(Y-\psi(X))\big]=\E_{M^0(i_*)}\big[\ell(f^0(X)+h(U,\epsilon_Y)-\psi(X))\big]=\E_{M^0(i_*)}\big[\ell(h(U, \epsilon_Y)-\overline{\psi}(X))\big].
\end{equation*}

Next, recall that $X^S\ci U\mid X^{S^c}$. Since $\epsilon_Y \ci (X, U)$, it also holds (using the properties of conditional independence) that $X^S\ci (U,\epsilon_Y) \mid X^{S^c}$ under $\P_{M^0(i_*)}$. Let $\mu$ denote the common dominating product measure and let $p_*$ denote the density corresponding to $\P_{M^0(i_*)}$ with respect to $\mu$. Then, by conditional independence it holds that
\begin{equation*}
    p_*(x^S, x^{S^c}, u, e)=p_*(x^S\mid x^{S^c})p_*(u, e\mid x^{S^c})p_*(x^{S^c}).
\end{equation*}
Expressing the expectation as an integral and applying Jensen's inequality once ($\ell$ is convex) we get that
\begin{align*}
    &\E_{M^0(i_*)}\big[\ell(h(U, \epsilon_Y)-\overline{\psi}(X))\big]\\
    &\quad=\int \ell(h(u, e)-\overline{\psi}(x^S,x^{S^c}))p_*(x^S\mid x^{S^c})p_*(u, e\mid x^{S^c})p_*(x^{S^c})\text{d}\mu(u, e, x^S,x^{S^c})\\
    &\quad\geq \int \ell\left( \int h(u, e)-\overline{\psi}(x^S,x^{S^c}) p_*(x^S\mid x^{S^c})\text{d}\mu(x^S)\right)p_*(u, e\mid x^{S^c})p_*(x^{S^c})\text{d}\mu(u, e, x^{S^c})\\
    &\quad= \int \ell\left( h(u, e)-\int\overline{\psi}(x^S,x^{S^c}) p_*(x^S\mid x^{S^c})\text{d}\mu(x^S)\right)p_*(u, e\mid x^{S^c})p_*(x^{S^c})\text{d}\mu(u, e, x^{S^c}).
\end{align*}
Let $\phi\in\mathcal{F}$ be a function that satisfies for all $x^{S^c}\in\supp(\P^{X^{S^c}}_{M^0(i_*)})$ that
\begin{equation*}
    \phi(x^{S^c})=\int\overline{\psi}(x^S,x^{S^c}) p_*(x^S\mid x^{S^c})\text{d}\mu(x^S).
\end{equation*}
Such a $\phi\in\mathcal{F}$ exists because $\mathcal{F}$ consists of all square-integrable functions. We then get that
\begin{align*}
    \E_{M^0(i_*)}\big[\ell(Y-\psi(X))\big]
    &\geq\int \ell\left( h(u, e)-\phi(x^{S^c})\right)p_*(u, e\mid x^{S^c})p_*(x^{S^c})\text{d}\mu(u, e, x^{S^c})\\
    &=\E_{M^0(i_*)}\big[\ell(h(U, \epsilon_Y)-\phi(X^{S^c}))\big]\\
    &=\E_{M^0(i_*)}\big[\ell(Y-(f^0(X)+\phi(X^{S^c})))\big].
\end{align*}
Combining this with \eqref{eq:minimum_part2} and defining $\overline{\phi}\in\mathcal{F}$ for all 
$x\in\R^d$ by $\overline{\phi}(x)\coloneqq f^0(x) + \phi(x^{S^{c}})$ leads to
\begin{equation}
    \label{eq:minimum_part3}
    \E_{M^0(i_*)}\big[\ell(Y-\overline{\phi}(X))\big]
    \leq\inf_{f_{\diamond}\in \mathcal{F}}\sup_{i\in\mathcal{I}}\E_{M^0(i)}\big[\ell(Y-f_{\diamond}(X))\big]+\varepsilon/2.
\end{equation}
Moreover, by construction it holds $\P_{M^0(i_*)}$-almost surely that 
\begin{equation}
    \label{eq:reduce_to_function}
    Y-\overline{\phi}(X)=h(U,\epsilon_Y)-\phi(X^{S^c}).
\end{equation}
Since $\P_{M^0(i_*)}$ is dominated by $\P_{M^0}$ and $\supp(\P_{M^0(i_*)}^{X})=\supp(\P_{M^0}^{X})$, \eqref{eq:reduce_to_function} also holds $\P_{M^0}$-almost surely. Finally, using that $X^{S^c}$ are not descendants of $Z$ and $Z$ is a source node, it holds that $Z\ci (X^{S^c}, U, \epsilon_Y)$ under $\P_{M^0}$, which implies
\begin{equation*}
    Z\ci Y-\overline{\phi}(X)\quad\text{under }\P_{M^0}.
\end{equation*}
Hence $\overline{\phi}\in\mathcal{F}_{\text{inv}}$.

\textbf{Step (2):} Since $\overline{\phi}\in\mathcal{F}_{\text{inv}}$ we can use \eqref{eq:i_does_not_matter1} and \eqref{eq:i_does_not_matter2} from the proof of Theorem~\ref{thm:invariant} together with \eqref{eq:minimum_part3} to get that
\begin{equation*}
    \sup_{i\in\mathcal{I}}\E_{M^0(i)}\big[\ell(Y-\overline{\phi}(X))\big]
    =\E_{M^0(i_*)}\big[\ell(Y-\overline{\phi}(X))\big]\leq
    \inf_{f_{\diamond}\in \mathcal{F}}\sup_{i\in\mathcal{I}}\E_{M^0(i)}\big[\ell(Y-f_{\diamond}(X))\big]
    +\varepsilon/2.
\end{equation*}
This proves \eqref{eq:eps_close_B}. Next, by the definition of the infimum there exists $f_{*}\in\mathcal{F}_{\text{inv}}$ satisfying
\begin{equation}
\label{eq:inf_bound}
    \E_{M^0}\big[\ell(Y-f_{*}(X))\big]\leq \inf_{f_{\diamond}\in\mathcal{F}_{\text{inv}}}\E_{M^0}\big[\ell(Y-f_{\diamond}(X))\big]+\epsilon/2.
\end{equation}
Now, assume that
\begin{equation*}
    \E_{M^0}\big[\ell(Y-f_{*}(X))\big]\leq \E_{M^0}\big[\ell(Y-\overline{\phi}(X))\big].
\end{equation*}
If this was not the case then \eqref{eq:eps_close_A} would be true and the proof would be complete. Again, applying Theorem~\ref{thm:invariant} and defining $c_{\min}\coloneqq\inf_{f_{\diamond}\in \mathcal{F}}\sup_{i\in\mathcal{I}}\E_{M^0(i)}\big[\ell(Y-f_{\diamond}(X))\big]$ this implies
\begin{equation*}
c_{\min}\leq
    \sup_{i\in\mathcal{I}}\E_{M^0(i)}\big[\ell(Y-f_{*}(X))\big]\leq \sup_{i\in\mathcal{I}}\E_{M^0(i)}\big[\ell(Y-\overline{\phi}(X))\big]\leq c_{\min}+\epsilon/2.
\end{equation*}
This implies that
\begin{equation*}
    \sup_{i\in\mathcal{I}}\E_{M^0(i)}\big[\ell(Y-\overline{\phi}(X))\big]\leq \sup_{i\in\mathcal{I}}\E_{M^0(i)}\big[\ell(Y-f_{*}(X))\big] + \epsilon/2= \EX_{M^0}\big[\ell(Y-f_{*}(X))\big] + \epsilon/2.
\end{equation*}
Combined with \eqref{eq:inf_bound} this proves \eqref{eq:eps_close_A}, which concludes the proof of Theorem~\ref{thm:generalization}.
\end{proof}

\section{Additional details - Conditional IV}
\label{sec:app_implmentation_civ}

We consider the following SCM.\footnote{
The results can be extended to other SCMs, too; we only show one for simplicity.
E.g., we could also allow for the edge between $Z$ and $W$ to point towards $W$, that is, for $Z$ to be a cause of $W$.}

\begin{minipage}{0.5\linewidth}
  \begin{align*}
    W&\coloneqq m(\epsilon_W,V,U)\\
    Z&\coloneqq q(W,V,\epsilon_Z)\\
    X&\coloneqq g(Z, W, U, \epsilon_X)\\
    V&\coloneqq \epsilon_{V}\\
    U&\coloneqq \epsilon_{U}\\
    Y&\coloneqq f(X) + h(W,U, \epsilon_Y),
  \end{align*}
  \hspace{0.5em}
\end{minipage}%
\hspace{-0.08\linewidth}
\begin{minipage}{0.44\linewidth}
  \resizebox{0.7\textwidth}{!}{
    \begin{tikzpicture}[scale=1.4]
      \tikzstyle{VertexStyle} = [shape = circle, minimum width = 3em,
      fill=lightgray] \Vertex[Math,L=Y,x=1,y=0]{Y}
      \Vertex[Math,L=X,x=-1,y=0]{X}
      \Vertex[Math,L=Z,x=-3,y=0]{Z}
      \Vertex[Math,L=W,x=-1,y=1.5]{W}
      \tikzset{EdgeStyle/.append style = {-Latex, line width=1}}
      \Edge[label=$f$](X)(Y)
      \Edge[label=$g$](Z)(X)
      \Edge[label=$q$](W)(Z)
      \Edge[label=$h$](W)(Y)
      \Edge[label=$g$](W)(X)
      \tikzstyle{EdgeStyle}=[bend right=50, line width=1, Latex-Latex, dashed]
      \Edge[label={\Large $U$}](X)(Y)
     \tikzstyle{EdgeStyle}=[bend left=50, line width=1, Latex-Latex, dashed]
     \Edge[label={\Large $V$}](Z)(W)
      \tikzstyle{EdgeStyle}=[bend left=50, line width=1, Latex-Latex, dashed]
      \Edge[label={\Large $U$}](W)(Y)
    \end{tikzpicture}}
\end{minipage} \\
where $(\epsilon_W,\epsilon_Z, \epsilon_{V}, \epsilon_{U}, \epsilon_X, \epsilon_Y)\sim Q$ are jointly independent noise variables, $Z\in\mathbb{R}^r$
are \emph{instruments}, $V\in\mathbb{R}^{s_1}$ and $U\in\mathbb{R}^{s_2}$ are unobserved variables, $W\in\mathbb{R}^t$ are observed \emph{covariates}, $X\in\mathbb{R}^d$ are \emph{predictors} and $Y\in\mathbb{R}$ is a \emph{response}. 
For simplicity, we additionally assume that $\E[W]=\E[Z]=\E[X]=\E[h(W,U, \epsilon_Y)]=0$. We consider models of the form $M=(m,q,g,f,h)$ and assume that either there is no edge from $V$ to $W$, i.e., $W=m(\epsilon_W,U)$ or no edge between $W$ and $U$, i.e., $W=m(\epsilon_W,V)$. As before, $M^0=(m^0,q^0,g^0,f^0,h^0)$ denotes the data generating model and $f^0(\cdot)=\phi(\cdot)^\top \theta^0$ for some $\theta^0 \in \Theta \subseteq \R^p$, with a (known) function $\phi: \R^d \rightarrow \R^p$.

Due to the unobserved confounding, it may be the case that $Y-f^0(X)$ and $Z$ are dependent and, as a result, restriction \eqref{eq:moment2} may not hold for the true model $M^0$. But as we assume that there is either no confounding between $Z$ and $W$ or between $W$ and $Y$, we can instead use the \textit{conditional independence restriction}
\begin{equation}
    Z \ci Y-\phi(X)^\top \theta \,| \,W
\end{equation}
to identify $f^0$. A corresponding sufficient and necessary identifiability condition for $f^0$ is 
\begin{equation*}
    Z \ci h(W,U,\epsilon_Y) + \phi(X)^\top \tau \,| \,W \, \, \implies \tau=0.
\end{equation*}
Under this condition we could proceed as for the unconditional independence restriction and estimate $f^0$ using a loss that is minimized if restriction~\eqref{condition: CIS 1} is satisfied, e.g., using a conditional independence measure.
We can avoid using a conditional independence restriction, in the following two special cases.

First, consider the case in which $W \ci V$ and, in addition, assume %
$q(W,\epsilon_z)=q_1(W)+\epsilon_Z$. Then we can use the 
following 
independence restriction
 \begin{equation}
     Z - q_1(W) \ci  Y-\phi(X)^\top \theta.
     \label{condition: CIS 1+}
 \end{equation}
Under our assumptions, $q^0_1(W)=E[Z|W]$ and therefore $q^0_1$ is identifiable. The corresponding identifiability condition becomes
\begin{equation}
    Z-q_1(W) \ci h(W,U,\epsilon_Y) + \phi(X)^\top \tau  \,\implies\, \tau=0.
    \label{condition: CIS 1 unique}
\end{equation}

Second, consider the case in which we instead assume that $W \ci U$ and that $h(W,U,\epsilon_Y)=h_1(W) +h_2(U,\epsilon_Y)$. Then we can use the 
independence restriction:
\begin{equation}
    Y-\phi(X)^\top \theta - \phi(W)^\top \gamma \, \ci \, (Z,W),
    \label{condition: CIS 2}
\end{equation}
where we assume that\footnote{For simplicity, we use the same basis $\phi$.} $h_1^0(\cdot)=\phi(\cdot)\gamma^0$ for $\gamma^0 \in \mathbb{R}^{t}$.
The identifiability condition becomes
\begin{align}
\begin{split}
    (Z,W) \ci h_2(U,\epsilon_Y) &+ \phi(X)^\top \tau_1 + \phi(W)^\top \tau_2 \,\implies\, \tau_1=0.
\end{split}
    \label{condition: CIS 2 unique}
\end{align}
This identifiability condition translates to a corresponding moment restriction and as in the case of unconditional IV, we can again show that the independence version leads to strictly stronger identifiability. 

\begin{proposition}[CIV - Identifiability of the independence restriction is strictly stronger]
\label{prop:civ}
    Consider a conditional IV model $M^0$ and assume that $\phi(\cdot)=(\phi_1(\cdot),\dots,\phi_p(\cdot))$ is a collection of basis functions such that $f^0(\cdot)=\phi(\cdot)^\top \theta^0$ for some $\theta^0 \in \mathbb{R}^{p}$ and $h_1^0(\cdot)=\phi(\cdot)^\top \gamma^0$ for some $\gamma^0 \in \mathbb{R}^{t}$. Then, if there exists $k\in \mathbb{N}$ and $\eta: \R^{(r+t)} \rightarrow \R^k$ such that $\theta^0$ is identifiable from the moment restriction
    \begin{equation*}
        \E [\eta(Z, W)(Y - \phi(X)^\top \theta - \phi(W)^\top \gamma)] = 0
        \text{,}
    \end{equation*}
    then $\theta^0$ is also identifiable from the conditional independence restriction \eqref{condition: CIS 2}. Furthermore, there exist a conditional IV model such that $\theta^0$ is identifiable from the independence restriction \eqref{condition: CIS 2} but not from the corresponding moment restriction. An example for that is the model
    \begin{align*}
        V \coloneqq \epsilon_V,\,
        U \coloneqq \epsilon_U,\,
        &W\coloneqq V + \epsilon_W,\,
        Z\coloneqq W + V + \epsilon_{Z},\, 
        X\coloneqq Z U + \epsilon_X \\
        & Y\coloneqq X + W + U + \epsilon_Y,
    \end{align*}
     where $\epsilon_Z, \epsilon_{V}, \epsilon_{U}, \epsilon_{W}, \epsilon_X$ and $\epsilon_Y$ are jointly independent standard Gaussian variables and we consider the base functions $\phi(X)=(X)$ and $\phi(W)=(W)$.
\end{proposition}

\begin{proof}

For the first claim, consider a $\theta^0$ that is not identifiable from restriction~\eqref{condition: CIS 2}. Then there exists a $\tau_1 \neq 0$ for which condition \eqref{condition: CIS 2 unique} is violated. By the independence statement for $\tau_1$, it holds for all $k \in \mathbb{N}$, functions $\eta: \R^{(r+t)} \rightarrow \R^k$ and some $\tau_2 \in \R$ that
\[\Cov [ \eta(Z,W), h(U,\epsilon_Y) + \phi(X)^\top \tau_1+\phi(W)^\top \tau_2] = 0.\] But as $(Z,W) \ci h(U,\epsilon_Y)$ it also holds that $\Cov [ \eta(Z,W), h(U,\epsilon_Y)]=0$. By the additivity of the covariance it follows that 
\[\Cov [ \eta(Z,W),\phi(X)^\top \tau_1+\phi(W)^\top \tau_2] = 0,\]
which is the identifiability condition corresponding to the moment restriction.

For the second claim, note that in our example SCM, $\mathbb{E}[X|(Z,W)]=0$. It follows that for all $k \in \mathbb{N}$ and functions $\eta: \R^{(r+t)} \rightarrow \R^k$,
$\Cov [ \eta(Z,W), X] =0$. We can therefore not uniquely identify the causal function $f^0(x)=x$ with the moment restriction.

We now consider condition \eqref{condition: CIS 2 unique}. Consider $K_{\tau}\coloneqq U + \epsilon_Y + \tau_1 (ZU + \epsilon_X) +\tau_2 W$. It holds that
\begin{align*}
\mathbb{E}[Z^2K_{\tau}^2]
&= \mathbb{E}[Z^2(U + \epsilon_Y + \tau_1 (ZU + \epsilon_X) +\tau_2 W)^2]\\ 
&= \mathbb{E}[Z^2U^2] + \mathbb{E}[Z^2\epsilon^2_Y] + \tau_1^2 \mathbb{E}[Z^4U^2] + \tau_1^2\mathbb{E}[Z^2\epsilon^2_X] +\tau_2^2\mathbb{E}[Z^2W^2]\\ 
&= 12 + 114\tau^2_1 + 30\tau^2_2,
\end{align*}
while
\begin{align*}
\mathbb{E}[Z^2]\mathbb{E}[K^2_{\tau}]
&= \mathbb{E}[Z^2]\mathbb{E}[(U + \epsilon_Y + \tau_1 (ZU + \epsilon_X) +\tau_2 W)^2]\\ 
&= \mathbb{E}[Z^2] (\mathbb{E}[U^2] +\mathbb{E}[\epsilon_Y^2] +\tau^2_1\mathbb{E}[Z^2U^2] +\tau^2_1\mathbb{E}[\epsilon^2_X] + \tau^2_2\mathbb{E}[W^2])\\
&= 6(2+7\tau^2_1 + 2 \tau^2_2)\\
&= 12+42\tau^2_1 + 12 \tau^2_2.
\end{align*}
These two terms are only equal if $72\tau_1^2+18\tau_2^2=0$, which requires that $\tau_1=0$.
Therefore, condition \eqref{condition: CIS 2 unique} holds, meaning we can identify the causal function $f^0(x)= x$.
\end{proof}

We now propose two procedures; one for each of the identifiabilty conditions we discussed. We conjecture that combining both leads to a doubly robust estimator.

For restriction~\eqref{condition: CIS 1+},
we propose the following estimation procedure. Given an i.i.d.\ sample $(x_i, y_i, z_i, w_i)_{i=1}^n$ of the variables $(X, Y, Z, W)$, we can first use any flexible non-parametric estimator to estimate $q_1^0$.
Based on our estimate $\hat{q}_1$ we then learn $\hat{f}$ as
\begin{align*}
    \argmin_{f \in \mathcal{F}} \widehat{\HSIC}((r_i^f , z^{\hat{q}}_i)_{i=1}^n; k_{R}, k_{Z^{\hat{q}_1}}),
\end{align*}
as our estimate for $f^0$, where $r_i^{f} = y_i - \hat{f}(x_i)$ and $z^{\hat{q}_1}_i = z_i - \hat{q}_1(w_i)$ and $\mathcal{F}$ is a function class. Here, $\widehat{\HSIC}$ is the same empirical $\HSIC$ estimator as in Section \ref{sec:algorithm}. 

For restriction~\eqref{condition: CIS 2}, we propose the following estimation procedure. Given an i.i.d.\ sample $(x_i, y_i, z_i, w_i)_{i=1}^n$ of the variables $(X, Y, Z, W)$, we estimate $f^0$ as the first entry of
\begin{align*}
    \argmin_{f \in \mathcal{F}, \,  k \in \mathcal{K}} \widehat{\HSIC}((r^{f,k}_i, (z_i,w_i))_{i=1}^n; k_{R_Y}, k_{Z,W}),
\end{align*}
where $r^{f,k}_i= y_1 - f(x_i) - k(w_i)$ and, $\mathcal{F}$ and $\mathcal{K}$ are function classes. Here, $\widehat{\HSIC}$ is again the same empirical $\HSIC$ estimator as in Section \ref{sec:algorithm}.

\newpage
\section{Algorithms: HSIC-X and HSIC-X-pen}\label{sec:app_algo}
We provide 
details for HSIC-X in Algorithm~\ref{alg:ind_iv}
and details for HSIC-X-pen in Algorithm~\ref{alg:ind_iv_pen}. We propose to choose the tuning parameter $\lambda$ for HSIC-X-pen as the largest possible value for which an HSIC-based independence test between the estimated residuals and the instruments is not rejected (see Section~\ref{sec:hsic_pen}).
\begin{algorithm}[ht]
\caption{HSIC-X}
\label{alg:ind_iv}
\SetAlgoLined
\KwIn{observations $(x_i, y_i, z_i)_{i=1}^n$, kernels $k_{R}$ and $k_{Z}$, batch size $m$, learning rate $\gamma$, number of gradient steps per cycle $k$, significance level $\alpha$ and maximum restarting times $t$}
\, Initialize a restart counter $\ell \leftarrow 0$ \\
Compute the bandwidth $\sigma_Z$ of $k_Z$ with median heuristic \\
\Repeat{$p \geq \alpha$ or $\ell \geq t$}{
\eIf{$\ell = 0$}{Initialize parameters $\theta$ at the OLS solution}{Randomly initialize parameters $\theta$}
    \Repeat{convergence}{
    \, Compute the residuals $r_i^{\theta} \coloneqq y_i - f_{\theta}(x_i)$ \\
    Compute the bandwidth $\sigma_{R^{\theta}}$ of $k_{R^{\theta}}$ with median heuristic \\
        \RepTimes{$k$}{
        \,\,  Sample a mini-batch $(x_j, y_j, z_j)_{j=1}^m$ \\
        Compute the residuals $r_j^{\theta} \coloneqq y_j - f_{\theta}(x_j)$ \\
        Compute the HSIC loss $\mathcal{L}(\theta) \coloneqq \tr(KHLH)$ \\
        Update parameters $\theta \leftarrow \theta - \gamma \nabla \mathcal{L}(\theta)$
        }
    }
\, Compute the residuals $r^{\theta}_i \coloneqq y_i - f_{\theta}(x_i)$ \\
Compute the p-value $p$ of the independence test between $(r_i^{\theta})_{i=1}^n$ and
$(z_i)_{i=1}^n$ \\
Update the counter $\ell \leftarrow \ell + 1$
} 
\KwOut{Final estimate $\tilde{f}_{\theta}(\cdot) \coloneqq f_{\theta}(\cdot) - \frac{1}{n}\sum_{i=1}^n y_i - f_{\theta}(x_i)$}
\end{algorithm}
\begin{algorithm}[ht!]
\caption{HSIC-X-pen}
\label{alg:ind_iv_pen}
\SetAlgoLined
\KwIn{observations $(x_i, y_i, z_i)_{i=1}^n$, kernels $k_{R}$ and $k_{Z}$, batch size $m$, learning rate $\gamma$, number of gradient steps per cycle $k$, significance level $\alpha$ and maximum restarting times $t$, penalization parameter $\lambda$, prediction loss $\ell$}
\, Initialize a restart counter $l \leftarrow 0$ \\
Compute kernel bandwidth $\sigma_Z$ of $k_Z$ with median heuristic \\
\Repeat{$p \geq \alpha$ or $l \geq t$}{
\eIf{$l = 0$}{Initialize parameters $\theta$ at the OLS solution}{Randomly initialize parameters $\theta$}
    \Repeat{convergence}{
    \, Compute the residuals $r_i^{\theta} \coloneqq y_i - f_{\theta}(x_i)$ \\
    Compute kernel bandwidth $\sigma_R$ of $k_R$ with median heuristic \\
        \RepTimes{$k$}{
        \,\,  Sample a mini-batch $(x_j, y_j, z_j)_{j=1}^m$ \\
        Compute the residuals $r_j^{\theta} \coloneqq y_j - f_{\theta}(x_j)$ \\
        Compute loss $\mathcal{L}(\theta) \coloneqq \lambda\sum_j \ell(r_j^{\theta})+(1-\lambda)\tr(KHLH)$ \\
        Update parameter $\theta \leftarrow \theta - \gamma \nabla \mathcal{L}(\theta)$
        }
    }
\, Compute the residuals $r^{\theta}_i \coloneqq y_i - f_{\theta}(x_i)$ \\
Compute the p-value $p$ of the independence test between $(r_i^{\theta})_{i=1}^n$ and
$(z_i)_{i=1}^n$ \\
Update the counter $l \leftarrow l + 1$
}
\KwOut{Final estimate $\tilde{f}_{\theta}(\cdot) \coloneqq f_{\theta}(\cdot) - \frac{1}{n}\sum_{i=1}^n y_i - f_{\theta}(x_i)$}
\end{algorithm}

\newpage

\section{Additional Experiment Details and Results}\label{sec:app_exp}
\subsection{Multi-dimensional Setting}{\label{sec:app_exp_multidim}}
We consider the following IV models in our experiments:
    \begin{equation*}
        M: \begin{cases}
        Z \coloneqq  \epsilon_{Z} \\
        U \coloneqq  \epsilon_{U} \\
        X \coloneqq  J Z^2 \epsilon_{X} + B Z + U \\
        Y \coloneqq  \beta^\top X - 2 U + \epsilon_Y,
        \end{cases}
\end{equation*}
where $\epsilon_U, \epsilon_Y \simiid \mathcal{N}(0, 1)$, $\epsilon_Z \sim \mathcal{N}(0, I_{d_Z})$, $\epsilon_X \sim \mathcal{N}(0, I_{d_X})$ are independent noise variables, $d_X$ and $d_Z$ represent the dimensions of $X$ and $Z$, $J$ and $B$ are $d_X \times d_Z$ matrices controlling the influence of the instruments $Z$ and $\beta$ is the causal parameters. In the experiment, $J$ is an $d_X \times d_Z$ all-ones matrix, $B_{i,j} \sim \text{Uniform}(-4, 4)$ and $\beta \sim \mathcal{N}(0, I_{d_X})$.

Figure~\ref{fig:exp_multi_dim} reports the MSE of each method as dimension ($d_{Z}$) of the instruments varies for some fixed predictor's dimensionality ($d_{X}$). When $d_{Z} < d_{X}$, the identifiability result suggests that 2SLS cannot consistently estimate the causal function (which is reflected by the high values of the MSE); while HSIC-X can still gain some improvement over the OLS solution. When $d_{Z} \geq d_{X}$, the performance of our method is on par with that of 2SLS.

\begin{figure}[t]
\centering
\includegraphics[width=0.48\textwidth]{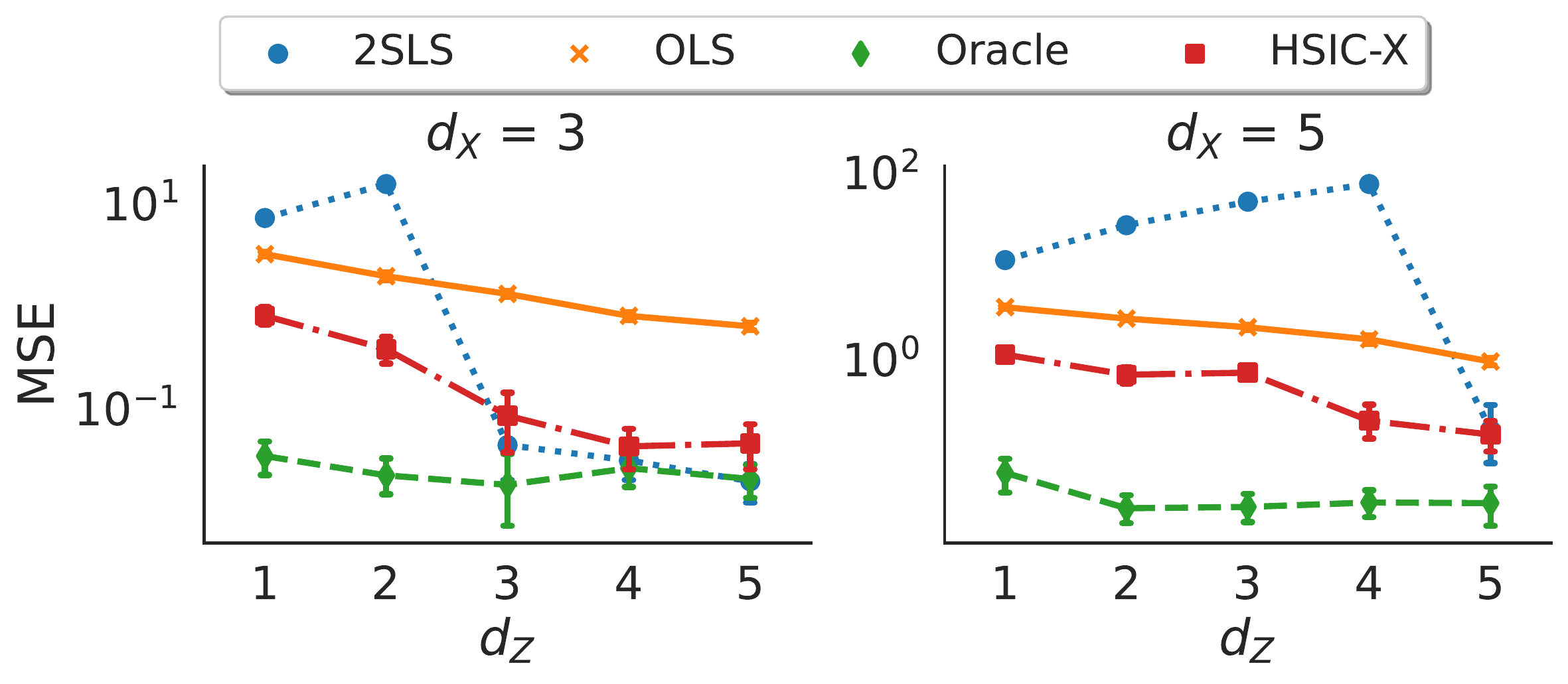}
\caption{MSEs of different estimators under varying size $d_Z$ of the instruments when (a) the predictors' dimension $d_{X}$ is 3, and (b) the predictors' dimension is 5. 2SLS is inconsistent and underperforms OLS when $d_Z < d_X$, while HSIC-X shows a substantial improvement over OLS in such settings.}
\label{fig:exp_multi_dim}
\end{figure}

\subsection{Known Basis Functions}\label{sec:app_exp_basis}
Figure~\ref{fig:estimated_causal_functions} shows some of the functions estimated by HSIC-X, OLS, 2SLS, along with the underlying causal function under various settings. In short, when the instrument has no effect on the mean of $X$ ($\alpha = 0$), 2SLS fails to produce sensible estimates because of the non-identifiability under the moment restriction. On the other hand, the proposed method (HSIC-X) can still identify the causal function thanks to the independence restriction and yields reasonable estimates in all of the settings.
\begin{figure*}[ht!]
    \centering
    \subfigure[$f^0$: Linear, $\P_{\epsilon_Z}$: Binary, $\alpha$: 0]{
         \centering
         \includegraphics[width=0.34\textwidth]{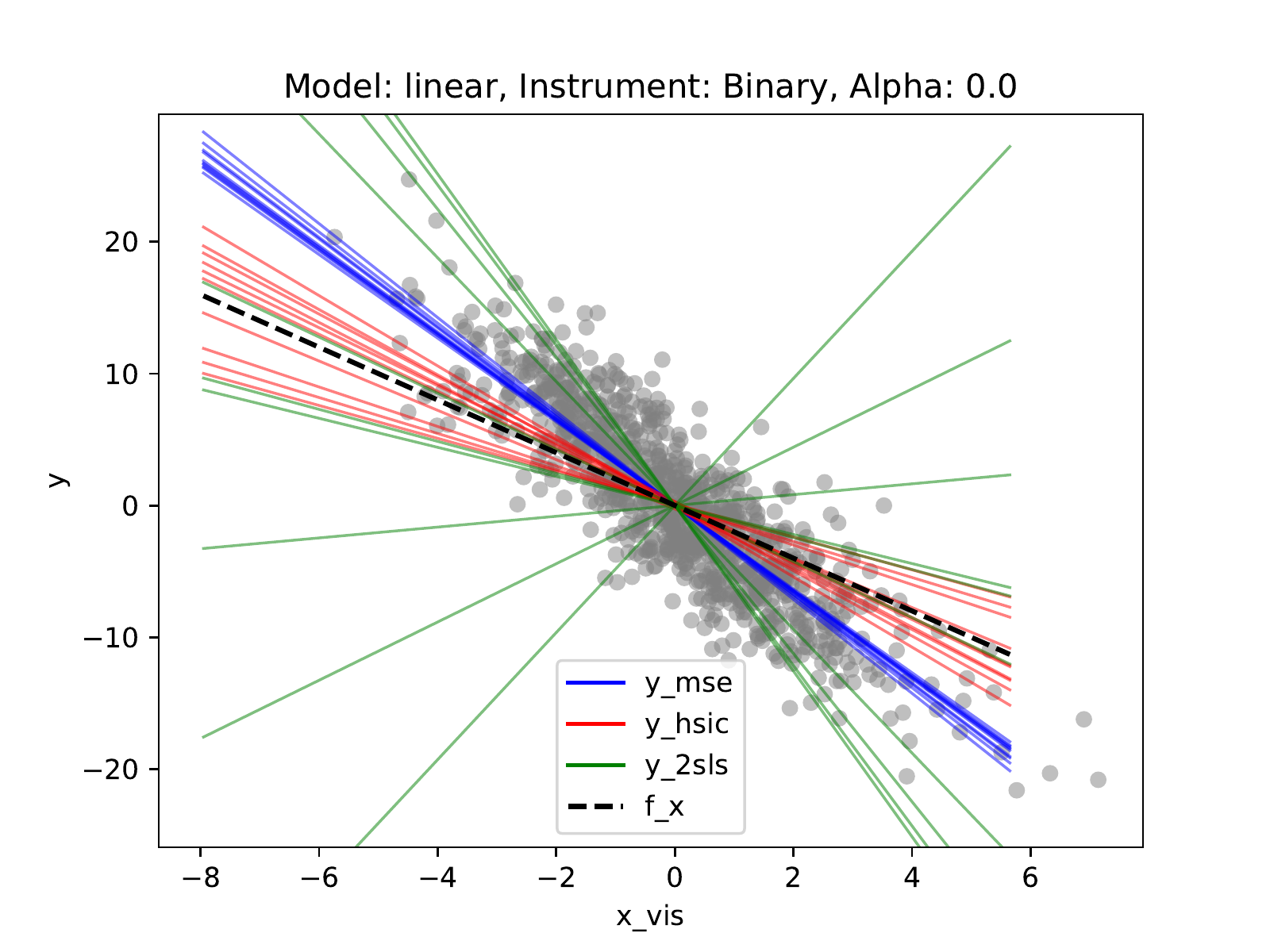}
    \label{fig:lin_binary11}
    }
    \subfigure[$f^0$: Linear, $\P_{\epsilon_Z}$: Binary, $\alpha$: 0.4]{
         \centering
         \includegraphics[width=0.34\textwidth]{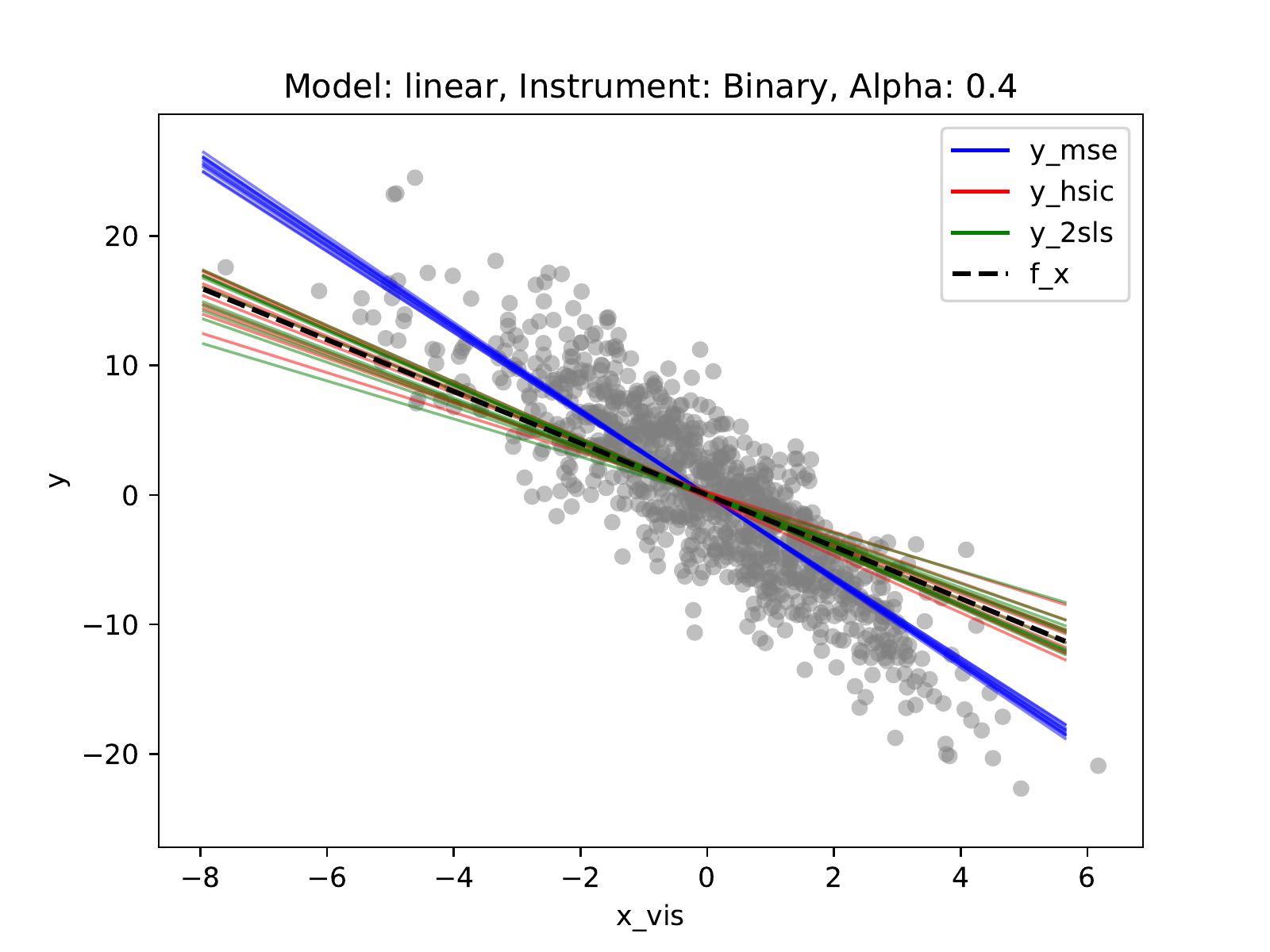}
    \label{fig:lin_binary12}
    }
    \subfigure[$f^0$: Linear, $\P_{\epsilon_Z}$: Gaussian, $\alpha$: 0]{
         \centering
         \includegraphics[width=0.34\textwidth]{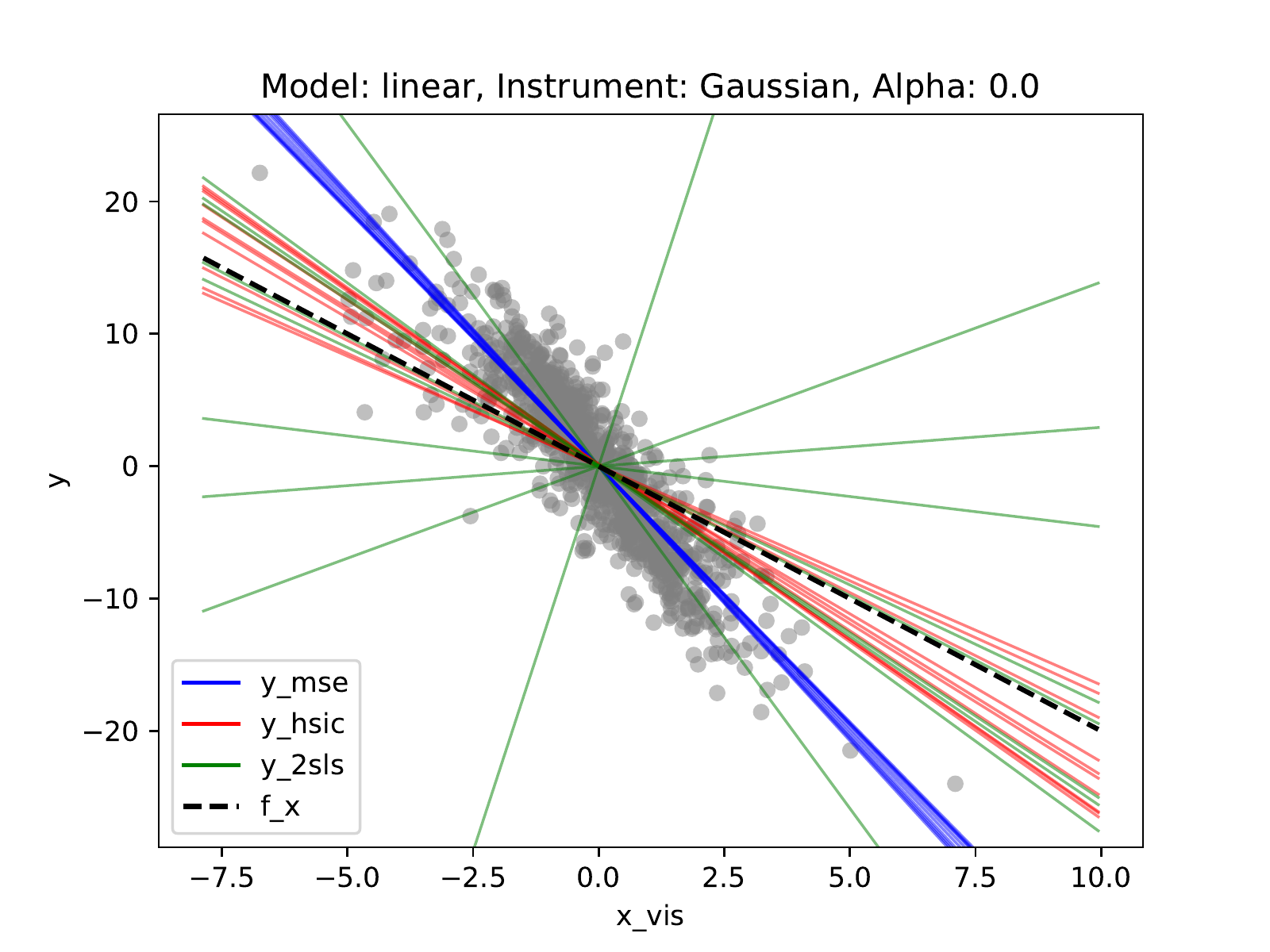}
    \label{fig:lin_binary13}
    }
    \subfigure[$f^0$: Linear, $\P_{\epsilon_Z}$: Gaussian, $\alpha$: 0.4]{
         \centering
         \includegraphics[width=0.34\textwidth]{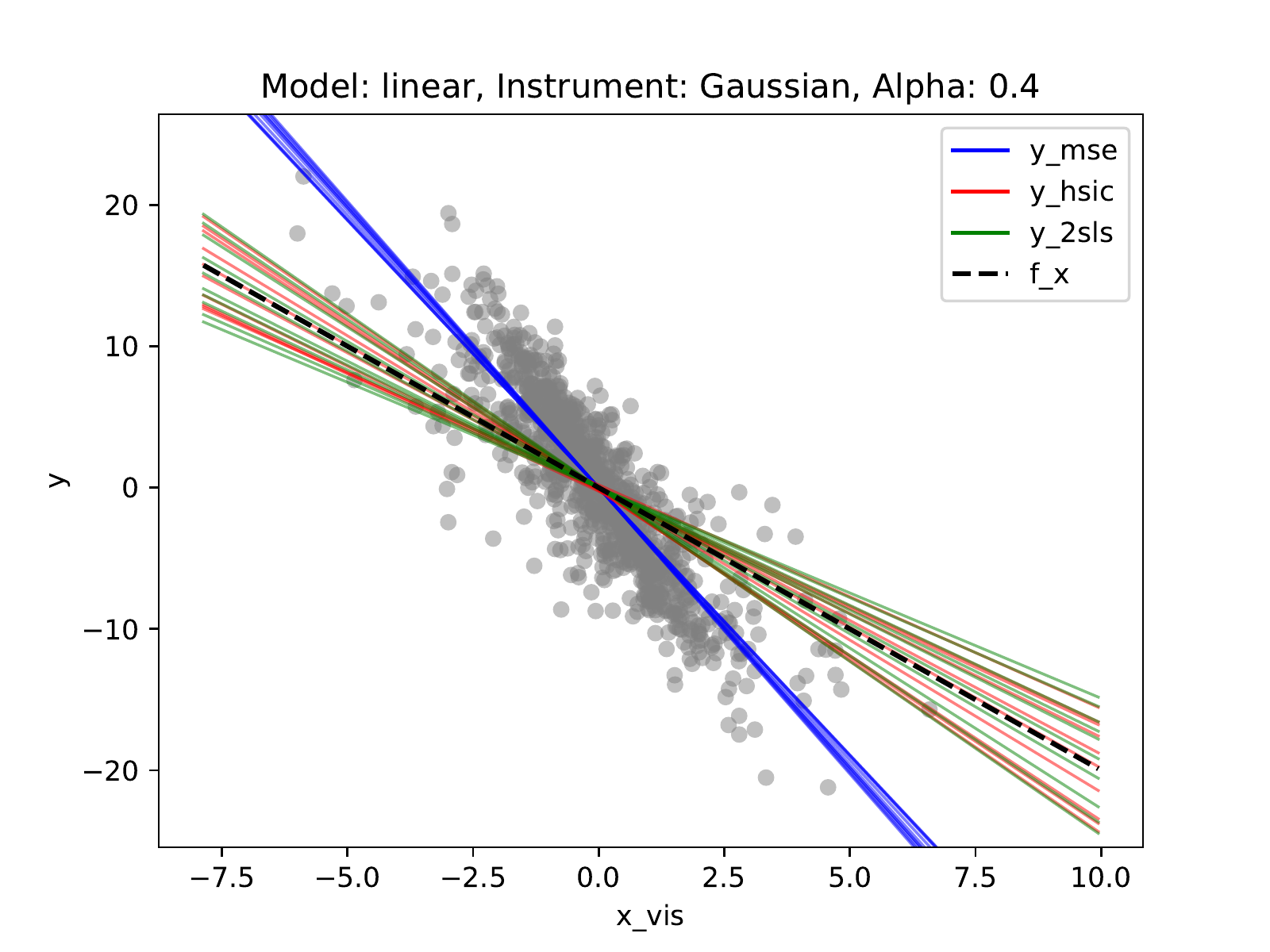}
    \label{fig:lin_binary14}
    }
    \subfigure[$f^0$: Nonlinear, $\P_{\epsilon_Z}$: Binary, $\alpha$: 0]{
         \centering
         \includegraphics[width=0.34\textwidth]{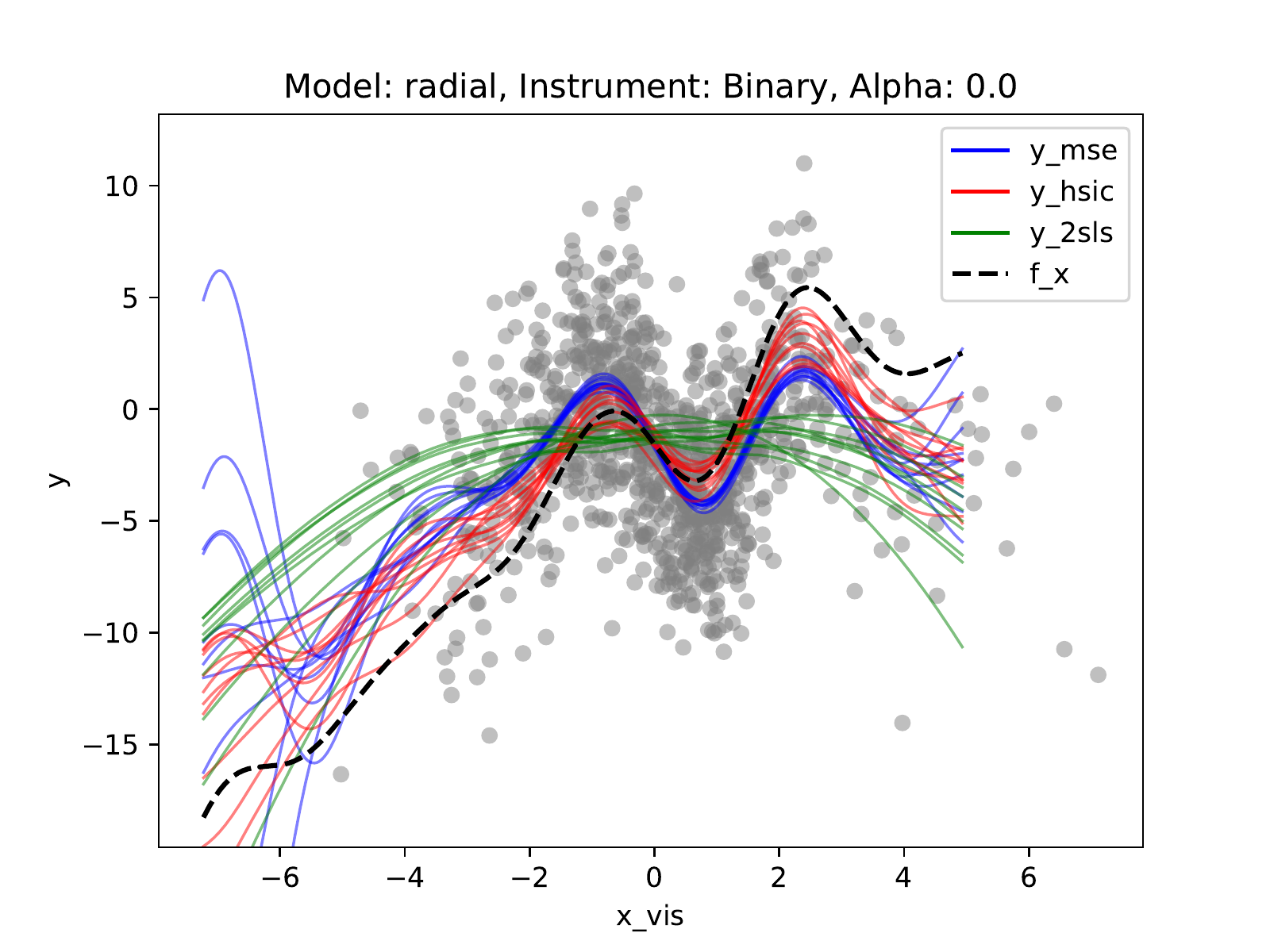}
    \label{fig:lin_binary15}
    }
    \subfigure[$f^0$: Nonlinear, $\P_{\epsilon_Z}$: Binary, $\alpha$: 1.0]{
         \centering
         \includegraphics[width=0.34\textwidth]{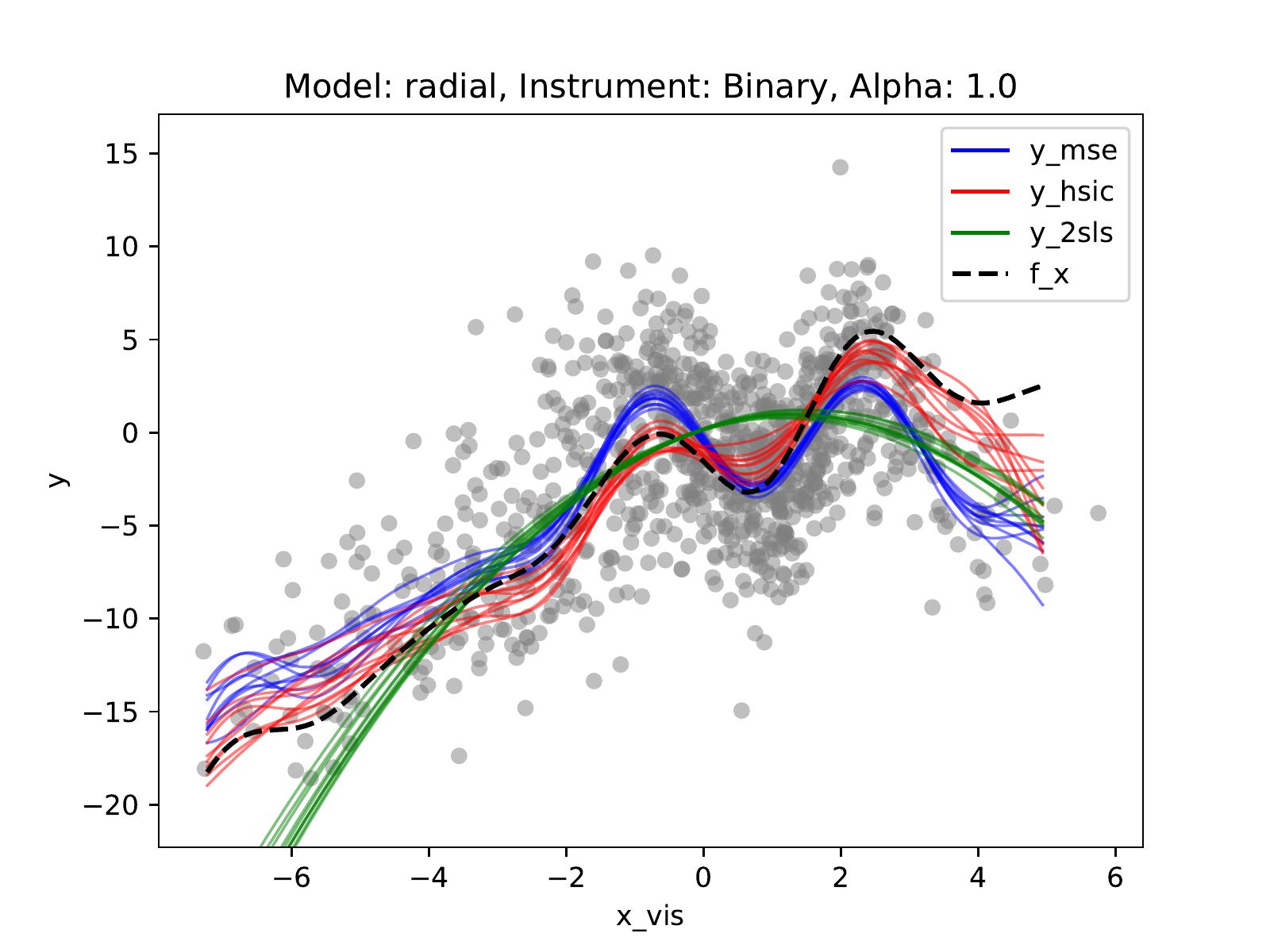}
    \label{fig:lin_binary16}
    }
    \subfigure[$f^0$: Nonlinear, $\P_{\epsilon_Z}$: Gaussian, $\alpha$: 0]{
         \centering
         \includegraphics[width=0.34\textwidth]{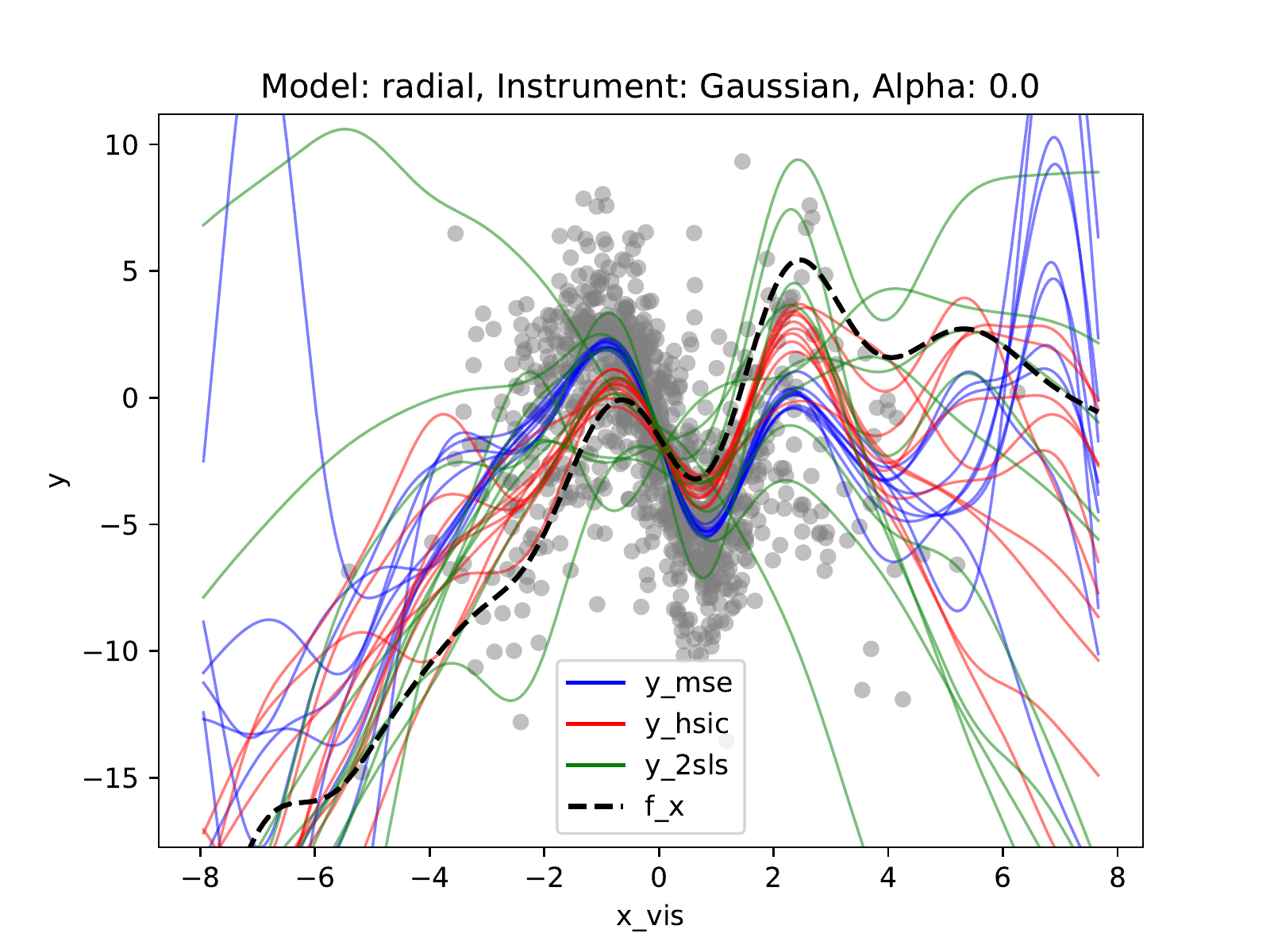}
    \label{fig:lin_binary17}
    }
    \subfigure[$f^0$: Nonlinear, $\P_{\epsilon_Z}$: Gaussian, $\alpha$: 1.0]{
         \centering
         \includegraphics[width=0.34\textwidth]{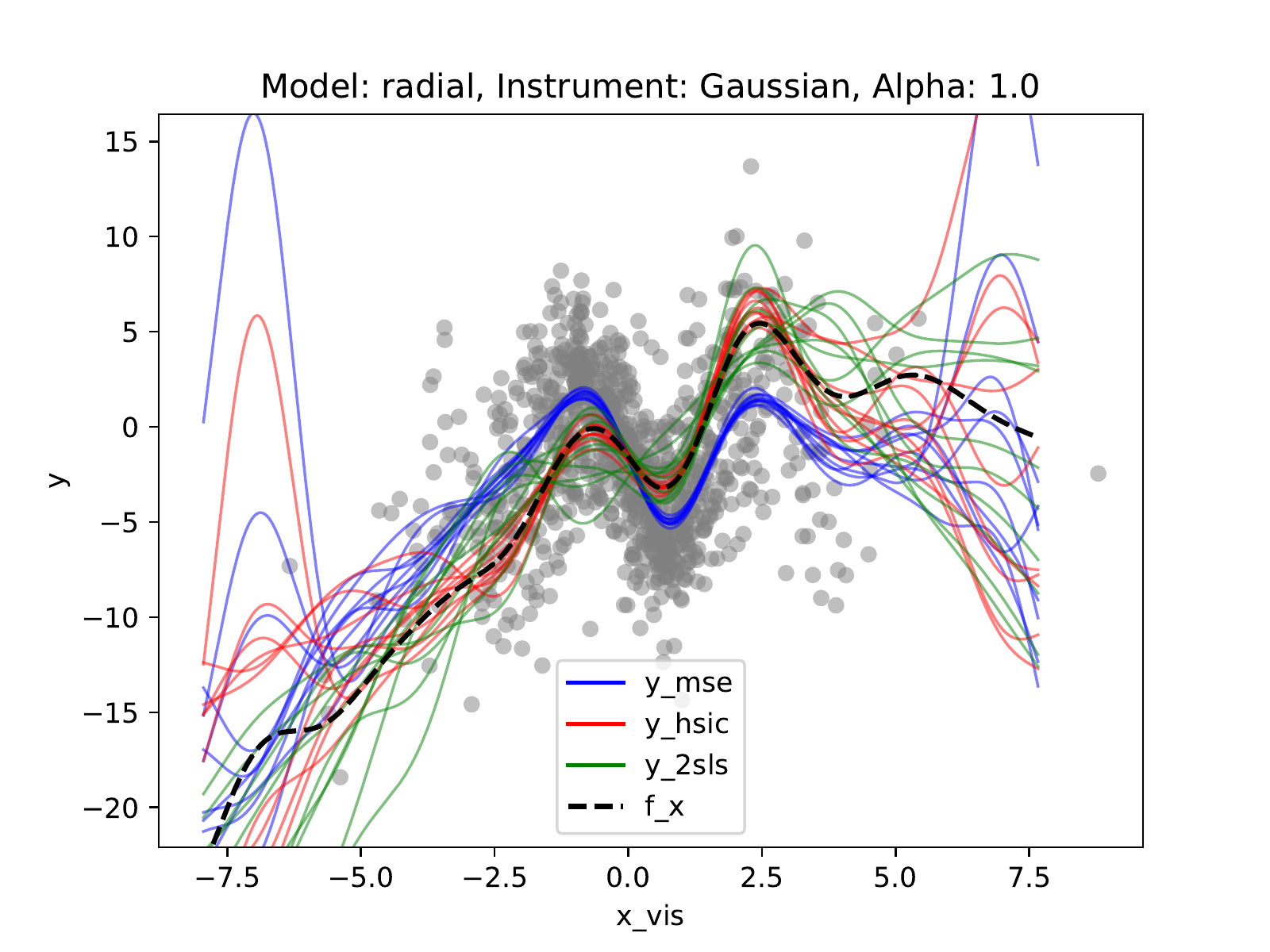}
    \label{fig:lin_binary18}
    }
\caption{Estimated causal functions with different estimators under varying $\alpha$ values from the experiment in Section~\ref{sec:sim_iv_basis} (the known basis functions setting). The legend reads as follows: "y\_mse" represents OLS, "y\_hsic" represents HSIC-X, "y\_2sls" represents 2SLS and "f\_x" represents the underlying causal function. Each line represents an estimated function from one simulation run. The grey dots are the observations drawn from the corresponding IV model.}
\label{fig:estimated_causal_functions}
\end{figure*}

\subsection{Distribution Generalization}\label{sec:app_exp_DG}
In Section~\ref{sec:sim_DG}, we consider both linear and nonlinear underlying causal functions. The linear function is defined as $f^0_{\text{lin}}(X_1, X_2) \coloneqq X_1 + X_2$, and the nonlinear function is defined as $f^0_{\text{nonlin}}(X) \coloneqq \sum_{j=1}^{10} w_j e^{-\frac{1}{3}\norm{X - c_j}_2^2}$, where $c_j$ is drawn from a uniform distribution over a two-dimensional grid $[-5, 5] \times [-5, 5]$  and $w_1,\dots,w_{10} \simiid \mathcal{N}(0,4)$. We employ the correct basis and neural network function classes in our method and the baselines. For the correct basis, we consider $\mathcal{F} \coloneqq \{ f(\cdot) = \phi(\cdot)^\top \theta \mid \theta \in \R^p \}$ with $p = 3$ and $\phi(x) = [1, x_1, x_2]$ in the linear case, and $p = 11$ and $\phi(x) = [1, e^{-\frac{1}{3}\norm{x - c_1}_2^2}, \dots, e^{-\frac{1}{3}\norm{x - c_{10}}_2^2}]$ in the nonlinear case. For the neural network function class, we use a neural network with one hidden layer of size $64$. We optimize our model using Adam optimizer with the learning rate of $0.005$ and batch-size of $256$, and use the R package `AnchorRegression' (\url{https://github.com/simzim96/AnchorRegression}) for the Anchor Regression baseline. Lastly, the tuning parameter $\gamma$ of Anchor Regression is set to $100$, and the tuning parameter $\lambda$ for HSIC-X-pen is chosen as the largest possible value for which an HSIC-based independence test between the estimated residuals and the instruments is not rejected (see Section~\ref{sec:hsic_pen}).

\end{document}